%% file: main.tex
\documentclass[12pt]{colt2026} 

\def\arxivbuild{}  % comment this line out for the camera-ready build

\ifdefined\arxivbuild
  \jmlrproceedings{}{}
  \jmlrvolume{}
  \jmlryear{}
  \jmlrworkshop{}
\fi
\title[Strongly Polynomial Time Complexity of Policy Iteration for $L_\infty$ Robust MDPs]{Strongly Polynomial Time Complexity of Policy Iteration\\ for $L_\infty$ Robust MDPs}
\usepackage{times}
\usepackage{amsfonts}
\usepackage{tikz}
\usepackage{listings}
\usepackage{xcolor}
\usepackage{algorithm}
\usepackage{enumerate,enumitem}
\usepackage{multicol,multirow}
\usepackage{hyperref}
\usepackage{paralist}
\usepackage{mathtools}
\usepackage{verbatim}
\usepackage[T1]{fontenc}
\usetikzlibrary{patterns, decorations.pathreplacing, arrows.meta, positioning}

\input{marcos}

\newtheorem*{mainthm}{Theorem}
\coltauthor{%
 \Name{Ali Asadi} \Email{ali.asadi@ist.ac.at}\\
 \Name{Krishnendu Chatterjee} \Email{krishnendu.chatterjee@ist.ac.at}\\
 \Name{Ehsan Goharshady} \Email{ehsan.goharshady@ist.ac.at}\\
 \Name{Mehrdad Karrabi} \Email{mehrdad.karrabi@ist.ac.at}\\
 \Name{Alipasha Montaseri} \Email{alipasha.montaseri@ist.ac.at}\\
 \Name{Carlo Pagano} \Email{carlo.pagano@concordia.ca}
}

\begin{document}

\maketitle

\begin{abstract}
Markov decision processes (MDPs) are a fundamental model in sequential decision making. Robust MDPs (RMDPs) extend this framework by allowing uncertainty in transition probabilities and optimizing against the worst-case realization of that uncertainty. In particular, $(s, a)$-rectangular RMDPs with $L_\infty$ uncertainty sets form a fundamental and expressive model: they subsume classical MDPs and turn-based stochastic games. We consider this model with discounted payoffs. The existence of polynomial and strongly-polynomial time algorithms is a fundamental problem for these optimization models. For MDPs, linear programming yields polynomial-time algorithms for any arbitrary discount factor, and the seminal work of Ye established strongly-polynomial time for a fixed discount factor. The generalization of such results to RMDPs has remained an important open problem. In this work, we show that a robust policy iteration algorithm runs in strongly-polynomial time for $(s, a)$-rectangular $L_\infty$ RMDPs with a constant (fixed) discount factor, resolving an important algorithmic question.

\end{abstract}


\input{intro}

\input{prelims}
\input{algorithm}

\input{rmc}

\input{rmdp}
\input{conclu}
\newpage
\pagebreak
% Acknowledgments---Will not appear in anonymized version
\acks{We thank Morteza Saghafian for his helpful suggestions when preparing the paper. This work was supported by Vienna Science and Technology Fund (WWTF), State of Lower Austria [Grant ID 10.47379/ICT25017], ERC CoG 863818
(ForM-SMArt) and Austrian Science Fund (FWF) 10.55776/COE12.}

\bibliography{bibliography}
\newpage
\pagebreak 
\appendix
\input{appendix}

\end{document}

%% file: marcos.tex
\newcommand{\states}{\mathcal{S}}
\newcommand{\trans}{P}
\newcommand{\cost}{c}
\newcommand{\R}{\mathbb{R}}
\newcommand{\E}{\mathbb{E}}

\newcommand{\val}{\boldsymbol{v}}
\newcommand{\discount}{\gamma}
\newcommand{\discountfactor}{\discount}
\newcommand{\mc}{M}
\newcommand{\rmc}{\mathcal{M}}
\newcommand{\uncert}{\mathcal{P}}
\newcommand{\uncertaintyset}{\uncert}
\newcommand{\envpol}{\tau}
\newcommand{\actions}{\mathcal{A}}
\newcommand{\Actions}{\actions}
\newcommand{\rmdp}{\mathcal{R}}
\newcommand{\agentpol}{\sigma}
\newcommand{\radius}{\delta}
\newcommand{\envpolopt}{\envpol^*}
\newcommand{\agentpolopt}{\agentpol^*}
\newcommand{\valopt}{\val^*}
\newcommand{\I}{\mathbb{I}}
\newcommand{\identitymatrix}{\I}
\newcommand{\Bellman}{\mathcal{T}}
\newcommand{\valvector}{\val}
\newcommand{\costvector}{\mathbf{c}}
\newcommand{\pvector}{\mathbf{p}}
\DeclareMathOperator*{\argmin}{argmin}
\DeclareMathOperator*{\argmax}{argmax}
\newcommand{\nominal}{\hat{P}}

\newcommand{\uvector}{\boldsymbol{u}}
\newcommand{\vvector}{\boldsymbol{v}}
\newcommand{\infnorm}[1]{\left\| #1 \right\|_\infty}

\newcommand{\States}{\states}
\newcommand{\lone}{\mathbf{L_1}}
\newcommand{\linf}{\mathbf{L_\infty}}
\renewcommand{\succ}{\textit{Succ}}
\newcommand{\receiver}{R}
\newcommand{\donor}{D}
\newcommand{\incomplete}{I}
\newcommand{\zero}{Z}

\renewcommand{\paragraph}[1]{{\smallskip\noindent\textbf{#1}}}
\newenvironment{proofsketch}{\paragraph{Proof Sketch.}}{\hfill$\blacksquare$}

%% file: intro.tex
\section{Introduction}

\paragraph{Robust Markov Decision Processes.}
Markov decision processes (MDPs) are a fundamental model in sequential decision making, where an agent interacts with a finite-state stochastic environment~\citep{DBLP:books/wi/Puterman94}. Solving MDPs focuses on minimizing the expected payoff with respect to some given cost function. In the context of MDPs, a common assumption is that the transition function is known. Although this assumption is plausible from a theoretical perspective, it is not always justified in practice since MDPs are constructed from data, and the transition functions are estimated with a level of uncertainty. This issue has motivated the study of \emph{Robust Markov decision processes} (RMDPs)~\citep{nilim2005robust,iyengar2005robust}. RMDPs weaken the assumption by only assuming knowledge of some uncertainty set containing the true transition function. The goal of solving RMDPs is to minimize the worst-case expected payoff with respect to all possible choices of transition functions belonging to the uncertainty set.

\paragraph{Discounted-sum payoff.} The agent aims to minimize a \emph{payoff function} which formally captures the desired behavior of the model. Discounted-sum payoffs are the most fundamental payoffs, i.e., given a cost function and discount factor, the payoff of a run (infinite trajectory of states and actions) is the sum of discounted costs of the run.

\paragraph{Policies and values.} Policies are recipes that define the choice of actions of the agent and the environment. They are functions that, given a finite sequence of states and actions (also known as history), return a distribution over actions. Given an RMDP and a discount factor, the payoff minimization is achieved over all possible choices of agent policies. The value of the agent at a state is the minimal payoff that the agent can guarantee against all policies of the environment.

\paragraph{$\linf$ uncertainty sets.} A central modeling choice in RMDPs is how we describe the uncertainty of transition probabilities. The literature considers several classes of uncertainty, including $(s, a)$-rectangular, $s$-rectangular, and more general non-rectangular models. In $(s, a)$-rectangular RMDPs, the environment selects transition probabilities separately for each state-action pair. In this paper, we adopt this rectangularity assumption for two reasons: (a)~it matches a common data-driven setting where each state-action pair $(s,a)$ has transition probabilities estimated from local samples, so uncertainty decomposes across state-action pairs; and (b) it preserves a dynamic-programming structure, yielding robust Bellman operators~\citep{nilim2005robust,iyengar2005robust}. We further assume $\linf$ uncertainty sets, i.e., for every state-action pair $(s, a)$, the true transition vector lies in an $\linf$ ball of radius $\delta_{s,a}$ around a nominal estimate $\nominal_{s,a}$.

\paragraph{Motivation.} The study of $(s, a)$-rectangular RMDPs with $\linf$ uncertainty sets is motivated by three reasons: (a)~they generalize several classical formalisms, including MDPs, and turn-based zero-sum stochastic games, where two adversarial players interact over a finite-state graph (see reduction in Appendix~\ref{app:stochastic-games}); (b) they provide a simple and interpretable bound on the \emph{maximum coordinate-wise} estimation error for the transition probabilities, which arises naturally from sample-based estimation~\citep{nilim2005robust,DBLP:journals/ai/GivanLD00,DBLP:journals/ai/DelgadoBDS16}; and (c) they keep the robust Bellman update computationally efficient~\citep{DBLP:conf/nips/BehzadianPH21}. Hence, $(s, a)$-rectangular RMDPs with $\linf$ uncertainty sets are a fundamental model.

\paragraph{Exact vs.~approximate value.} While approximation algorithms for RMDPs
are well-studied, computing the exact value is important for the following reasons.
First, the exact value is a fundamental theoretical question: in MDPs with
discounted-sum payoffs, strongly-polynomial additive approximation is easy via
value iteration, and the seminal work of~\cite{DBLP:journals/mor/Ye11a}
establishes a strongly-polynomial algorithm for the exact value. Second, for
MDPs and RMDPs with discounted-sum objectives, additive approximation guarantees
are easy to obtain via value iteration, however, multiplicative-factor
approximation (guaranteeing a high fraction of the optimal value) is not easy,
and solving for the exact value gives a stronger guarantee.

\paragraph{Previous computational results.} In the study of MDPs and related optimization problems, the study of efficient algorithms is a central problem. For example, MDPs with discounted-sum objectives can be solved in polynomial-time via linear-programming~\citep{d1963probabilistic,Derman1972FiniteStateMDP}. 
It has been a long-standing problem to obtain a strongly polynomial-time algorithm where the number of arithmetic operations is polynomial independent of the bit length.
For the special case of fixed discount factor, the seminal work of~\cite{DBLP:journals/mor/Ye11a} obtains the first strongly-polynomial time algorithm, which has also been extended to
stochastic games by~\cite{hansen2013strategy}. Efficient (i.e., polynomial and strongly-polynomial time) algorithms for RMDPs with $\linf$ uncertainty sets is a fundamental algorithmic problem. The only claimed result in the literature is a polynomial-time bound from~\cite{DBLP:conf/nips/BehzadianPH21} (which claims without proof that this follows from~\cite{hansen2013strategy}).
However, this claim was an oversight as the mentioned technique does not yield the desired result for two reasons. First, \cite{hansen2013strategy} studies turn-based stochastic games, and the only known reduction, presented in~\cite{DBLP:conf/ijcai/ChatterjeeGK0Z24}, is from \emph{$(s,a)$-rectangular} RMDPs (not $s$-rectangular RMDPs) to turn-based stochastic games. Second, even for $(s,a)$-rectangular RMDPs, the reduction is not of polynomial size in general. For each uncertainty set, it introduces an action for each corner of the corresponding polytope, and for $\linf$ uncertainty sets, the number of actions can be exponential. 

\paragraph{Main open problem.} Hence, the existence of polynomial and more importantly strongly-polynomial time algorithms for 
RMDPs with $\linf$ uncertainty sets and a fixed discount factor is an important open question in the study of sequential decision making under uncertainty.

% \paragraph{Open problem.} Existing analyses of robust value and policy iteration for $(s, a)$-rectangular RMDPs with discounted-sum payoffs are developed for approximation schemes. Moreover, the best known time complexity for robust policy iteration is only polynomial-time (in the input encoding), and does not yield a strongly polynomial bound. Informally, an algorithm is strongly polynomial if its number of arithmetic operations is bounded by a polynomial in the problem dimensions, e.g., numbers of states, actions, and discount factor, and does not depend on the numerical magnitudes or bit-length of the input data, e.g., rewards and uncertainty sets. This leaves open whether robust policy iteration admits a strongly polynomial runtime guarantee for $\linf$ uncertainty sets.

\paragraph{Our contributions.} In this work, we affirmatively answer the above open problems. In particular, we show that a robust policy iteration algorithm 
terminates in strongly polynomial time, which is a generalization of the result of~\cite{hansen2013strategy}.

\paragraph{Technical contributions.} 
\begin{compactitem}
    \item We first consider the robust policy iteration algorithm on robust Markov chains (\texttt{RMC-PI}, Algorithm~\ref{algo:RMCPI}). These are RMDPs where the agent has a single action in every state.
    \begin{compactitem}
        \item Firstly, we introduce a novel potential function that tracks the effects of changing the optimal policy within the uncertainty set with respect to probability mass transfers allowed by the homotopy algorithm (Algorithm~\ref{algo:homotopy}).
        \item We then show bounds to relate the policy values and the defined potential function. Moreover, we prove Lemma~\ref{lemma:combinatorial_lemma}, which is a novel combinatorial result over the number of most significant bits in the binary representation of unitary signed subset sums of a finite set of real numbers. We derive Lemma~\ref{lemma:combinatorial_lemma} as a consequence of the more general Theorem \ref{thm: main1} (See Appendix~\ref{app:lemma:combinatorial_lemma}). 
        We present here also the effective strengthening given by Theorem \ref{thm: main2}, which was subsequently and independently obtained, shortly after we established our Theorem \ref{thm: main1}, by a custom mathematics research agent, named \emph{Aletheia} \footnote{\emph{Aletheia} is a custom mathematics research agent built upon Gemini Deep Think, details at https://github.com/google-deepmind/superhuman/tree/main/aletheia.}, and built upon Gemini Deep Think at Google DeepMind under the lead of Tony Feng. Although our own Theorem \ref{thm: main1} is sufficient for Lemma \ref{lemma:combinatorial_lemma}, we decided to report the stronger version as in Theorem \ref{thm: main2}, both for its own intrinsic mathematical interest as well as to keep track of a contribution of AI to mathematics.  For more details on these generalizations, we refer to Appendix \ref{app:lemma:combinatorial_lemma}. For more details on \emph{Aletheia}, we refer to \cite{Aletheia}.
        \item Finally, by the end of Section~\ref{sec:rmc}, we use the combinatorial result together with proved bounds to show that \texttt{RMC-PI} runs in strongly polynomial time given an RMC with $n$ states and a constant discount factor.
    \end{compactitem}
    \item Secondly, we focus on the robust policy iteration algorithm on RMDPs (\texttt{RMDP-PI}) and deploy a similar pipeline as above to show that $\texttt{RMDP-PI}$ runs in strongly polynomial time given an RMDP with $n$ states, $m$ actions, and a constant discount factor.
\end{compactitem}

\paragraph{Related works.} In this work we focus on the exact value computation for RMDPs with $\linf$ uncertainty sets. There are many related works that consider a more general model, but for the value approximation problem. $(s, a)$-rectangular RMDPs with discounted-sum payoffs have been extensively studied in the literature (see \cite{DBLP:conf/birthday/SuilenBB0025}), starting with the seminal works of~\cite{nilim2005robust,iyengar2005robust}. In the general setting of $(s, a)$-rectangular uncertainty, both~\cite{nilim2005robust,iyengar2005robust} presented \emph{robust value iteration} methods to compute the approximate value. Moreover, \cite{iyengar2005robust} introduced a \emph{robust policy iteration} algorithm to approximate the value of the same $(s, a)$-rectangular model. Later, \cite{DBLP:journals/informs/KaufmanS13} studied another variant of robust policy iteration for discounted-sum $(s, a)$-rectangular RMDPs and showed that robust policy iteration converges to an approximate value. However, these works do not present polynomial or strongly-polynomial bounds for the exact value computation, which is the focus of this work.

%% file: prelims.tex
\section{Preliminaries and Model Description} \label{sec:prelims} 

\paragraph{Notation.} For arbitrary set $S$, we use $\Delta(S)$ to show the set of all probability distributions over $S$, and use $2^S$ to show the set of all subsets of $S$. Moreover, we use $size(n)= \lceil \log_2 n \rceil$ and $size(\frac{p}{q})=size(p)+size(q)$ as the bit-size of integer $n$ and rational number $\frac{p}{q}$.

\paragraph{Markov Chains.} A Markov chain (MC) is a tuple $\mc =(\states,\cost,\trans)$ where $\states = \{1,\dots,|\states|\}$ is the set of states, $\cost\colon \states \to \R$ specifies the cost of visiting each state, and $\trans\colon \states \to \Delta(\states)$ denotes the transition probabilities. The transition probability function $\trans$ can be seen as an $n \times n$ matrix where $\trans_{s,t} = \trans(s)[t]$. Similarly, the cost function $\cost$ can be presented as a vector in $\R^n$. The semantics of MCs are defined over sequences $\pi = \langle \pi_0, \pi_1, \dots\rangle$ of states. Using $\Pr[\pi_{i+1}|\pi_i] = \trans_{\pi_{i},\pi_{i+1}}$, the cylinder construction of~\cite{DBLP:books/daglib/0020348} yields a probability distribution over the set of all infinite sequences of states. Given a discount factor $\discount$, the value function $\val_\mc(s)$ for each state $s \in \states$ is defined as the expected discounted total cost of sequences that start at $s$:

\begin{equation} \label{Eq:RMCValue}
    \val_\mc(s) = \E_\pi\big[\sum_i \discount^i \cost(\pi_i)|\pi_0=s\big] = \big[(\I - \discount \trans)^{-1} \cost\big](s)
\end{equation}

\paragraph{Robust MCs.} Robust MCs (RMCs) are an extension of MCs by considering a particular uncertainty in the transition probabilities. Formally, an RMC is a tuple $\rmc = (\states,\cost,\succ,\uncert)$ where $\states$ and $\cost$ are the same as in MCs, $\succ\colon \states \to 2^\states$ indicates the set of possible successors of each state, and $\uncert \colon \states \to 2^{\Delta(\states)}$ is the uncertain transition function satisfying $\uncert(s) \subseteq {\Delta(\succ(s))}$. The semantics of RMCs are defined with respect to environment policies. A stationary environment policy is a function $\envpol\colon \states \to \Delta(\states)$ mapping each state to an element of its uncertainty set. Formally, $\envpol(s) \in \uncert(s)$ for each $s \in \states$. By fixing an environment policy $\envpol$, the uncertainty of the RMC is resolved, yielding an MC $\rmc^\envpol = (\states,\cost,\uncert^\envpol)$. The value of a policy $\envpol$ is then defined as the value of its induced MC, i.e. $\val^\envpol_\rmc(s) = \val_{\rmc^\envpol}(s)$. The goal of the environment is to choose a policy that maximizes the expected discounted cost. Formally, an optimal policy $\envpol^*$ is a policy that satisfies $\val^{\envpolopt}_\rmc(s) = \max_{\envpol} \val^{\envpol}_\rmc(s)$. We drop the subscript $\rmc$ whenever it is clear from the context and use $\valopt$ to denote $\val^{\envpolopt}$.

\paragraph{Robust Markov Decision Processes.}  It is possible to further extend RMCs into robust Markov Decision Processes (RMDPs) by considering a finite action space for the agent. Formally, an RMDP is a tuple $\rmdp = (\states,\cost,\actions,\succ,\uncert)$ where $\states$ and $\cost$ are defined similarly as before, $\actions$ is a finite set of actions, $\succ\colon \states\times \actions \to 2^\states$, and $\uncert\colon \states \times \actions \to 2^{\Delta(\states)}$ is the uncertain transition function where $\uncert(s,a) \subseteq {\Delta(\succ(s,a))}$ for each state-action pair $(s,a)$. The semantics of RMDPs are defined with respect to an agent policy $\agentpol\colon \states \to \Delta(\actions)$ that maps each state to a distribution over the action-set $\actions$. By fixing such an agent policy, the agent's choices are resolved, leaving an RMC $\rmdp^\agentpol = (\states, \cost,\succ^\agentpol, \uncert^\agentpol)$. The value of an agent policy is defined as the value of its induced RMC, i.e. $\val^\agentpol_\rmdp(s) = \val_{\rmdp^\agentpol}(s)$. The agent's objective is to find a policy that minimizes the worst-case expected cost. The optimal agent policy is then denoted by $\agentpolopt$ and satisfies: 
\begin{equation} \label{eq:value_definition_rmdp}
    \val^*_\rmdp(s) = \val_\rmdp^{\agentpolopt}(s) = \min_\agentpol \valopt_{\rmdp^\agentpol}(s) = \min_\agentpol \max_\envpol \val_{\rmdp^{\agentpol,\envpol}}(s) =  \min_\agentpol \max_\envpol \big[(\I - \discount \uncert^{\agentpol,\envpol})^{-1} \cost\big] (s)    
\end{equation}

When $\rmdp$ is clear from the context, we use $\val^{\agentpol}(s)$ and $\val^{\agentpol,\envpol}(s)$ to denote $\val_{\rmdp^\agentpol}(s)$ and $\val_{\rmdp^{\agentpol,\envpol}}(s)$, respectively, for agent policy $\agentpol$ and environment policy $\envpol$. We use $\agentpolopt$ and $\envpolopt$ for the optimal agent policy and the best environment response, respectively, satisfying $\val^*_\rmdp(s) = \val^{\agentpolopt}(s) = \val^{\agentpolopt,\envpolopt}(s)$.   %Note that RMCs are a special class of RMDPs where the action set $\actions$ is a singleton.%, hence the agent has only a trivial policy.

\paragraph{Assumption 1. (Rectangularity)} The uncertainty considered in this paper is also known as $(s,a)$-rectangular uncertainty since the environment can observe the agent's action before choosing a transition function from the uncertainty set. Other classes of uncertainty, such as $s$-rectangularity, have been studied in the literature (\cite{DBLP:conf/birthday/SuilenBB0025}) but are not the focus of our work.

\paragraph{Assumption 2. ($\mathbf{L_\infty}$ Uncertainty Sets)} In this work, we focus on RMCs and RMDPs whose uncertainty sets are defined as $\linf$ balls around a nominal transition function. Formally, in an $L_\infty$ RMC $\rmc$, for each $s \in \states$, there exists non-negative real number $\radius_s$ and a nominal probabilistic transition function $\nominal_s \in \Delta(\states)$ such that $\uncert(s) = \{\trans \in \Delta(\states)| \Vert \trans - \nominal_s \Vert_\infty \leq \radius_s \}$. Similarly, in an $L_\infty$ RMDP, for each $s \in \states$ and $a \in \actions$, there exists a non-negative real number $\radius_{s,a}$ and nominal transition function $\nominal_{s,a}$ such that $\uncert(s,a) = \{\trans \in \Delta(\states)| \Vert \trans - \nominal_{s,a} \Vert_\infty \leq \radius_{s,a} \}$. For the rest of the paper, we assume all the uncertainty sets are of this type.

\paragraph{Policy Types.} The policies considered in the above definitions are positional (their value depends only on the current state of the model). In general, policies can also depend on the history of visited states and chosen actions. However, it is a classical result that in discounted-sum RMDPs, positional (memoryless) deterministic optimal policies exist~(\cite{iyengar2005robust,nilim2005robust}).  Hence, in the rest of the paper, all the policies are assumed to be positional and deterministic, i.e. an agent policy in an RMDP is a map $\agentpol\colon \states \to \actions$ and an environment policy is a map $\envpol \colon \states \times \actions \to \Delta(\states)$ where $\envpol(s,a) \in \uncert(s,a)$.

\paragraph{Strongly Polynomial Algorithms.} An algorithm is strongly polynomial if the following two conditions hold: (i) The number of arithmetic operations is bounded by a polynomial in the number of input variables and constraints, and (ii) the size of all numbers during the computation is bounded by a polynomial in the input size.

Recall that an algorithm runs in polynomial time if its runtime is polynomial in the input (bit-length) size. Hence, in polynomial time algorithms, the runtime of the algorithm might depend on the bit-size of the input coefficients. In contrast, in strongly polynomial algorithms, the number of arithmetic operations is independent of the numerical values of the coefficients, and any dependence on the input coefficients is limited to their unavoidable contribution to the input encoding length.

\paragraph{Problem Statement.} Given an RMC $\rmc$ or an RMDP $\rmdp = (\states,\cost,\actions,\succ,\uncert)$ that satisfies Assumptions~1~and~2, algorithms such as {\em value iteration} and {\em policy iteration} are designed for computing the value function $\val$ together with the optimal agent policy $\agentpol^*$ and environment policy $\envpol^*$ that correspond to $\val$. Our goal in this work is to study the policy iteration method (explained in Section~\ref{sec:homotopyoverview}) and to show that it is strongly polynomial.

% \paragraph{Remark.} It is possible to consider a set of {\em active} successors $\succ(s,a)$ for each state-action pair $(s,a)$ and allow uncertainty only over the probability of transitioning to the active successors, i.e. $\uncert(s,a) = \{\trans \in \Delta(\states)| \Vert \trans - \nominal_{s,a} \Vert_\infty \leq \radius_{s,a} \wedge \forall s' \notin \succ(s,a)\colon \trans(s')=0  \}$. We note that all our theories can be easily extended to models with active successor sets, however, we focus our attention to the simpler models to simplify our notations and have better readability. Through the rest of this paper, we assume $|\states|=n$, and $|\actions|=m$.

%% file: algorithm.tex
\section{Overview: Robust Policy Iteration} \label{sec:algo}

We state our two main results; their proofs span Sections~\ref{sec:rmc} and~\ref{sec:rmdp}.

\begin{mainthm}\textnormal{\textbf{(Strongly Polynomial \texttt{RMC-PI})}}
    \texttt{RMC-PI} is strongly polynomial given an RMC with $n$ states and constant discount factor $\discountfactor$.
\end{mainthm}

\begin{mainthm}\textnormal{\textbf{(Strongly Polynomial \texttt{RMDP-PI})}}
    \texttt{RMDP-PI} is strongly polynomial given an RMDP $\rmdp$ with $n$ states, $m$ actions and constant discount factor $\discountfactor$.
\end{mainthm}

\paragraph{Policy Iteration.} Policy iteration is a classical method for computing the value function in Markovian models (\cite{DBLP:books/wi/Puterman94}). Algorithm~\ref{algo:RMCPI} shows the policy iteration algorithm for RMCs. Intuitively, given an RMC $\rmc$ and discount factor $\discount$, the algorithm starts from an arbitrary initial environment policy $\envpol^0$. Then, in each iteration, it first evaluates the current policy by computing $\val^{\envpol^t}$ and then greedily improves $\envpol^t$ into $\envpol^{t+1}$ in order to maximize the expected cost based on the computed value function.

\input{algo-PIRMC}

The main difference between \texttt{RMC-PI} in Algorithm~\ref{algo:RMCPI} and the classical policy iteration on Markov Decision Processes (\cite{DBLP:books/wi/Puterman94}) lies in the policy improvement step (also known as computing the Bellman operator), where in MDPs, the $\argmax$ is taken over the finite set of actions, but in RMCs it is taken over the possibly {\em infinite} uncertainty set $\uncert$.

\paragraph{Homotopy Algorithm.} \label{sec:homotopyoverview} Given that $\uncert$ is an $\linf$ uncertainty set, the policy improvement can be computed using linear programming. However, solving an LP instance is inefficient, its runtime being $O(n^{3.5}L)$, which depends on the bit precision $L$ of the input system. In \cite{DBLP:conf/nips/BehzadianPH21}, authors present a {\em homotopy} algorithm for computing the policy improvement, which runs in $O(n \log n)$. Algorithm~\ref{algo:homotopy} illustrates their approach adapted to our notations.

\input{algo-homotopy}

Intuitively, given the RMC $\rmc$, state $s \in \states$, and the current policy's value vector $\val$, the algorithm first sorts the successor states with respect to their $\val$ value so that $\val(s_1) \geq \val(s_2) \geq \dots \geq \val(s_{|\succ(s)|})$. It then proceeds by adapting a two-pointer technique for increasing the probability of transitioning to states with higher values while decreasing the probability of transitioning to states with lower values. This is done by keeping two indices $hi$ and $lo$. While $hi < lo$, the algorithm increases $\pvector_{s_{hi}}$ and decreases $\pvector_{s_{lo}}$ as much as possible. If $\pvector_{s_{hi}}$ could not be increased further, i.e. by reaching 1 or $\nominal_{s,s_{hi}}+\radius(s)$, then $hi$ is increased by one. Similarly, if $\pvector_{s_{lo}}$ could not be decreased further, i.e. by reaching 0 or $\nominal_{s,s_{lo}}-\radius(s)$, then $lo$ is decreased by one. See \cite{DBLP:conf/nips/BehzadianPH21} for correctness arguments.

\paragraph{Properties of Homotopic Policies.} The distribution $\pvector$ generated by Algorithm~\ref{algo:homotopy} possesses a specific structure that becomes useful later in our analysis. Essentially, there are four disjoint sets of states $\receiver_\pvector$, $\donor_\pvector$, $\zero_\pvector$ and $\incomplete_\pvector$, where: 
\begin{compactitem}
    \item for $s' \in \receiver_\pvector$, $\pvector_{s'} = \min(\nominal_{s,s'}+\radius(s),1)$,
    \item for $s' \in \donor_\pvector$, $\pvector_{s'} = \nominal_{s,s'} - \radius(s)$,
    \item for $s' \in \zero_\pvector$, $\nominal_{s,s'} \leq \radius(s) \; \wedge \; \pvector_{s'} = 0$, and
    \item $|\incomplete_\pvector| \leq 1$ and for $s' \in \incomplete_\pvector$, $\nominal_{s,s'} - \radius(s) < \pvector_{s'} = 1 - \sum_{s'' \neq s'}{\pvector_{s''}} < \nominal_{s,s'} + \radius(s)$.
\end{compactitem}
Intuitively, $\receiver_\pvector$ states are {\em receiving} $\radius(s)$ probability mass. This mass is {\em donated} by donor states in $\donor_\pvector$ and $\zero_\pvector$. Moreover, there is at most one state in $\incomplete_\pvector$ that receives/donates {\em incompletely}. These properties immediately show that there are at most exponentially $\mathcal{O}(n\cdot 3^n)$ many different homotopic policies that Algorithm~\ref{algo:homotopy} can choose from.

\paragraph{Strongly Polynomial Runtime of \texttt{RMC-PI}.} Our primary objective is to demonstrate that the policy iteration method in Algorithm~\ref{algo:RMCPI} runs in strongly polynomial time. Specifically, this means the number of policy improvement iterations is bounded by a polynomial independent of the cost function, nominal probabilities, or uncertainty set radii. We formally prove this result in Section~\ref{sec:rmc}.

\paragraph{Policy Iteration for RMDPs.} We can extend policy iteration to compute optimal policies for RMDPs, utilizing Algorithm~\ref{algo:RMCPI} as a key component. Algorithm~\ref{algo:RMDPPI} outlines this process. The method begins with an arbitrary agent policy, $\agentpol^0$. In the $t$-th iteration, it first computes the optimal environment counter-policy, $\envpol^t$, against $\agentpol^t$ (using Algorithm~\ref{algo:RMCPI}) and evaluates the value function $\val^t$ for $\rmdp^{\agentpol^t,\envpol^t}$. Finally, it uses $\val^t$ to improve the agent policy $\agentpol^t$ greedily.

\input{algo-PIRMDP}

\paragraph{Strongly Polynomial Runtime of \texttt{RMDP-PI}.} Our final objective is to establish a strongly polynomial upper bound on the runtime of the \texttt{RMDP-PI} algorithm. In this context, the bound depends on the number of states, the available actions, and the discount factor. As with the RMC case, we formally prove this result in Section~\ref{sec:rmdp}.

%% file: algo-PIRMC.tex
\begin{algorithm}[t]
\caption{Policy Iteration For Robust Markov Chains (\texttt{RMC-PI})}
\label{algo:RMCPI}
\SetKwInOut{Input}{Input}
\SetKwInOut{Output}{Output}
\SetKwRepeat{RepeatUntil}{Repeat}{Until}

\Input{RMC $\rmc = (\states,\cost,\succ,\uncert)$, discount factor $\discount$}
\Output{The optimal environment policy $\envpol^*$}

% \BlankLine % Add a blank line for separation, optional
$\envpol^0 \gets \texttt{arbitrary environment policy}$

$t \leftarrow 0$\;

\RepeatUntil{$\envpol^t = \envpol^{t-1}$}{
    \tcp{Policy Evaluation:}
    $\val^t \gets (\I - \discount \uncert^{\envpol^t})^{-1}\cost$
    
    $t \leftarrow t + 1$\;
    
    \tcp{Policy Improvement via Homotopy Algorithm~\ref{algo:homotopy}:}
    \ForAll{$s \in \states$}{
    $\envpol^{t}(s) \gets \argmax\limits_{\pvector \in \uncert(s)}{\cost(s) + \discount \pvector^\top \val^{t - 1}}$}
}

\Return $\envpol^t$\;

\end{algorithm}

%% file: algo-homotopy.tex
\begin{algorithm}[t]
\caption{Homotopy Algorithm for Policy Improvement}
\label{algo:homotopy}
\SetKwInOut{Input}{Input}
\SetKwInOut{Output}{Output}
\SetKwRepeat{RepeatUntil}{Repeat}{Until}

\Input{RMC $\rmc = (\states,\cost,\succ,\uncert)$, successor set $\succ(s)$, nominal transition matrix $\nominal_s$,   state $s\in \states$, radius $\radius(s)$ for the uncertainty set, value vector $\val$}
\Output{Distribution $\pvector \in \uncert(s)$ maximizing $\pvector ^\top \val$}
$\states' \gets \succ(s)$

Sort states $s_1, \dots, s_{|\states'|}$ of $\states'$ so that $\val(s_1) \geq \val(s_2) \geq \dots \geq \val(s_{|\states'|})$

$\pvector \gets \nominal_s$

$hi \gets 1, lo \gets |\states'|$

$b_{hi} \gets \radius(s), \quad b_{lo} \gets \radius(s)$

\While{$hi < lo$}{
$d_{hi} \gets \min(b_{hi},\, 1-\pvector_{s_{hi}})$

$d_{lo} \gets \min(b_{lo},\, \pvector_{s_{lo}})$

$t \gets \min(d_{hi},\, d_{lo})$

$\pvector_{s_{hi}} \gets \pvector_{s_{hi}} + t$

$\pvector_{s_{lo}} \gets \pvector_{s_{lo}} - t$

$b_{hi} \gets b_{hi} - t$

$b_{lo} \gets b_{lo} - t$

\eIf{$b_{hi} = 0$ \textbf{or} $\pvector_{s_{hi}} = 1$}{
    $hi \gets hi+1$\;
    $b_{hi} \gets \radius(s)$
}
{
    $lo \gets lo -1$\;
    $b_{lo} \gets \radius(s)$
}
}
\Return $\pvector$

\end{algorithm}

%% file: algo-PIRMDP.tex
\begin{algorithm}[t]
\caption{Policy Iteration For Robust Markov Decision Process \texttt{RMDP-PI}}
\label{algo:RMDPPI}
\SetKwInOut{Input}{Input}
\SetKwInOut{Output}{Output}
\SetKwRepeat{RepeatUntil}{Repeat}{Until}

\Input{RMDP $\rmdp = (\states,\cost,\actions,\succ,\uncert)$, discount factor $\discount$}
\Output{The optimal policy $\agentpol^*$}

% \BlankLine % Add a blank line for separation, optional
$\agentpol^0 \gets \texttt{arbitrary agent policy}$

$t \leftarrow 0$\;

\RepeatUntil{$\agentpol^t = \agentpol^{t-1}$}{
    \tcp{Policy Evaluation, Invoke Algorithm \ref{algo:RMCPI}:}
    $\envpol^{t} \gets \texttt{RMC-PI}(\rmdp^{\agentpol^t},\discount)$
    
    $\val^t \gets (I - \discount \uncert^{\agentpol^t,\envpol^t})^{-1}\cost$
    
    $t \leftarrow t + 1$\;
    
    \tcp{Policy Improvement:}
    \ForAll{$s \in \states$}{
    $\agentpol^{t}(s) \gets \argmin\limits_{a \in \actions}\max\limits_{\pvector \in \uncert(s,a)}{\cost(s) + \discount \pvector^\top \val^{t - 1}}$}
}

\Return $\agentpol^t$\;

\end{algorithm}

%% file: rmc.tex
\section{Robust Markov Chains} \label{sec:rmc}

% \pashacomment{Define $\valvector$ and $\costvector$ in Preliminaries.}

In this section, we focus our attention on RMCs and their corresponding policy iteration algorithm \texttt{RMC-PI} presented in Algorithm~\ref{algo:RMCPI}. Specifically, we assume \texttt{RMC-PI} is executed on an RMC $\rmc = (\states,\cost,\succ,\uncert)$ with an $\linf$ uncertainty set, $|\states|=n$ and a constant discount factor $\discountfactor$.

Let $\rmc$ be the input RMC in an execution of \texttt{RMC-PI}. The Bellman operator for computing the optimal values in $\rmc$ is defined as: 
\begin{equation} \label{Eq:BellmanRMC}
    (\Bellman \valvector)_s = \costvector_s + \discountfactor \max_{\pvector \in \uncertaintyset(s)} \pvector^\top \valvector.
\end{equation}
We first note that \texttt{RMC-PI} converges monotonically to the unique fixed-point of $\Bellman$ at an exponential rate. This is a classical result in the analysis of Markov models (\cite{DBLP:books/wi/Puterman94,iyengar2005robust}), presented in Lemma~\ref{lemma:contractive_value_iteration} (proved in Appendix~\ref{app:lemma:bellman}) and Lemma~\ref{lemma:consecutive_policy_iteration_bound} (proved in Appendix~\ref{app:lemma:convergence}).

\begin{lemma} \phantomsection\label{lemma:contractive_value_iteration} 

% \label{corollary:unique_optimal}
    The following statements hold:
    \begin{compactitem}
        \item For every $\vvector, \uvector \in \R^n$, $\infnorm{\Bellman \uvector - \Bellman \vvector} \leq \discountfactor \, \infnorm{\uvector - \vvector}.$
        \item For every environment policy $\envpol$ with value vector $\val^\envpol$, $\Bellman \val^{\envpol} \succcurlyeq \val^{\envpol}$.
        \item There exists a unique vector $\valvector^* \in \R^n$ such that $\Bellman \valvector^* = \valvector^*$.
    \end{compactitem}
\end{lemma}

Using the improvement properties established in Lemma \ref{lemma:contractive_value_iteration}, we show that the sequence of values generated by \texttt{RMC-PI} is non-decreasing and converging at an exponential rate:

\begin{lemma}\phantomsection\label{lemma:consecutive_policy_iteration_bound} \label{lemma:policy_iteration_less_than_value_iteration} \label{lemma:policy_iteration_exponential_upperbound}
    Let $\envpol^0,\envpol^1,\dots$ be the sequence of policies generated by \texttt{RMC-PI} (Algorithm \ref{algo:RMCPI}), and let $\valvector^t$ be the value of $\envpol^t$. Then the following statements hold for each $t$:
    \begin{compactitem}
        \item $\valvector^{t+1} \succcurlyeq \valvector^{t}$
        % \item $\valvector^{t+1} \succcurlyeq \Bellman \valvector^t$
        \item $\infnorm{\valvector^t - \valvector^*} \leq \gamma^t \infnorm{\valvector^0 - \valvector^*}$
    \end{compactitem}
\end{lemma}

Lemma \ref{lemma:policy_iteration_exponential_upperbound} establishes that policy iteration converges exponentially fast to the optimal policy. Given the exponential upper-bound on the number of homotopic policies that Algorithm~\ref{algo:homotopy} can return, Lemma~\ref{lemma:policy_iteration_exponential_upperbound} immediately implies that \texttt{RMC-PI} terminates in finite time and computes the optimal environment policy $\envpolopt$ together with its value function $\val^*$. 

To show that \texttt{RMC-PI} is strongly polynomial, we must show two statements: (i) it terminates within a polynomial number of arithmetic operations, and (ii) the bit-size of all intermediate values remains polynomial in the input size. Both conditions are satisfied if the total number of iterations of \texttt{RMC-PI} is bounded by a polynomial in $|\states|$. Specifically, the former is implied by the fact that both policy evaluation and policy improvement are themselves polynomial. The latter holds because each iteration recomputes $\val^t = (\identitymatrix - \discountfactor \uncert^{\envpol^t})^{-1} \costvector$ from scratch: the entries of $\uncert^{\envpol^t}$ belong to a fixed finite set determined by $\nominal$ and $\radius$ (by the structure of $\envpol^t$), so their bit-size is polynomial in the input bit-size, and matrix inversion therefore yields a $\val^t$ whose bit-size is polynomial in the input bit-size and $\mathit{size}(\discountfactor)$, uniformly in $t$. So, we focus on showing that the number of iterations of \texttt{RMC-PI} is polynomial in $|\states|$.

To this end, our goal is to bound the change in $\infnorm{\val^{t+1} - \val^*}$ compared to $\infnorm{\val^t - \val^*}$ and then use the bound to obtain bounds on the maximum number of iterations before termination. Given the structure of the policies returned by the homotopy algorithm (Algorithm~\ref{algo:homotopy}), we introduce a specific {\em potential function} $f_\envpol(s,s',s'')$ to track how much the value of $\envpol$ can be improved by donating some probability mass from $\uncert^\envpol_{s,s''}$ to $\uncert^\envpol_{s,s'}$:

\begin{definition}[\textbf{Potential Function}] \phantomsection\label{def:potential_function}
Given a policy $\envpol$, for any $s,s',s'' \in \States$ define:
\[
    f_\envpol(s,s',s'') 
    := \min\!\left(
        \uncert^*_{s,s'} - \uncert^\envpol_{s,s'}, \;
        \uncert^\envpol_{s,s''} - \uncert^*_{s,s''}
    \right) 
    \left( \valvector^*_{s'} - \valvector^*_{s''} \right),
\]
% where $s'$ represents a state where the optimal transition $\uncert^*$ places more mass than $\envpol$, and $s''$ represents a state where $\envpol$ places excess mass.
\end{definition}

Note that $f_\envpol(s,s',s'')$ can take negative values; the maximizing triple used in Lemmas~\ref{lemma:upperbound_potential} and~\ref{lemma:lower_upper_combination} is, however, always non-negative (e.g., taking $s' = s''$ yields $f_\envpol = 0$).

Intuitively, the $\min$ fragment in the definition of the potential function specifies the probability mass being donated to $\uncert^\envpol_{s,s'}$ from $\uncert^\envpol_{s,s''}$, and the $\val^*$ fragment quantifies the change in values.

We now establish a lower bound on the optimality gap in terms of the potential function. This result (proved in Appendix~\ref{app:lemma:lowerbound_potential}) connects the potential function to the actual difference in values. %Specifically, if a policy $\envpol$ allows for a feasible mass transfer that improves the value, the global value difference must be at least proportional to that improvement.

\begin{lemma} \phantomsection\label{lemma:lowerbound_potential}
    For every policy $\envpol$, and states $s,s',s'' \in \States$ it holds that
    $
        \valvector^*_s - \valvector^\envpol_s \ge \discountfactor \, f_\envpol(s,s',s'').
    $ 
\end{lemma}

% \begin{lemma}
%     Let $\envpol$ be a policy. Let $s, s', s'' \in \States$ be states that maximize the quantity $f_\envpol(s,s',s'')$ subject to the constraint that increasing $\uncert^\envpol_{s,s'}$ by
%     \[
%         \epsilon := \min\!\left( \uncert^*_{s,s'} - \uncert^\envpol_{s,s'}, \; \uncert^\envpol_{s,s''} - \uncert^*_{s,s''} \right)
%     \]
%     and decreasing $\uncert^\envpol_{s,s''}$ by $\epsilon$ keeps the row $\uncert^\envpol_{s,\cdot}$ within the $\lone$-ball of radius $\radius_s$.
%     Then:
%     \[
%          \infnorm{(\uncert^* - \uncert^\envpol) \valvector^*} \leq poly(n) f_\envpol(s,s',s'').
%     \]
% \end{lemma}

% We complement the lower bound by deriving an upper bound on the value error. 
Next, we derive an upper-bound on the distance between $\val^\envpol$ and $\val^*$ based on the potential of $\envpol$. We view the difference between the optimal transition matrix and the current policy's matrix as a set of probability mass transfers. By bounding the impact of these transfers using the potential function, we link the global error (distance to optimal $\val^*$) back to the local slack (potential function).

\begin{lemma} \phantomsection\label{lemma:upperbound_potential}
    Let $\envpol$ be a policy. Let $s, s', s'' \in \States$ be the states that maximize $f_\envpol(s,s',s'')$. Then:
    \[
         \infnorm{\valvector^* - \valvector^\envpol} \leq \frac{n^2 \discountfactor}{1 - \discountfactor} f_\envpol(s,s',s'').
    \]
\end{lemma}

\begin{proofsketch}
    We first show that 
    $\infnorm{\valvector^* - \valvector^\envpol} \leq \discountfactor \infnorm{(\identitymatrix - \discountfactor\uncert^\envpol)^{-1}} \infnorm{(\uncert^* - \uncert^\envpol) \valvector^*} $ 
    and proceed by showing an upper-bound on each factor of the RHS. (i) We show: $\infnorm{(\identitymatrix - \discountfactor\uncert^\envpol)^{-1}} \leq \frac{1}{1 - \discountfactor} $ by using the geometric series sum. (ii) We then show that $\infnorm{(\uncert^* - \uncert^\envpol) \valvector^*} \leq  n^2 f_\envpol(s, s', s'') $ by decomposing $\uncert^* - \uncert^\envpol$ into a series of at most $n^2$ {\em mass transfer} operations each contributing at most $f_\envpol(s,s',s'')$ to the LHS. The full proof is presented in Appendix~\ref{app:lemma:upperbound_potential}. 
\end{proofsketch}

As a byproduct of Lemma~\ref{lemma:lowerbound_potential} and Lemma~\ref{lemma:upperbound_potential}, we can now obtain a bound between $\infnorm{\val^* - \val^\envpol}$ and $\infnorm{\val^* - \val^{\envpol'}}$ based on the potential functions of $\envpol$ and $\envpol'$:

\begin{lemma} \phantomsection\label{lemma:lower_upper_combination}
    Assume $\envpol$ and $\envpol'$ are two policies. Let $s, s', s'' \in \States$ be the states that maximize $f_\envpol(s,s',s'')$.
    Suppose further that $\envpol'$ satisfies:
    \[
        \min\left( \uncert^*_{s,s'} - \uncert^{\envpol'}_{s,s'}, \; \uncert^{\envpol'}_{s,s''} - \uncert^*_{s,s''} \right) \geq \frac{1}{2} \min\left( \uncert^*_{s,s'} - \uncert^\envpol_{s,s'}, \; \uncert^\envpol_{s,s''} - \uncert^*_{s,s''} \right).
    \]
    Then $\infnorm{\valvector^* - \valvector^{\envpol'}} \geq \frac{1 - \discountfactor}{2n^2} \infnorm{\valvector^* - \valvector^\envpol}$.
\end{lemma}

\begin{proofsketch}
    We chain the upper bound of Lemma~\ref{lemma:upperbound_potential} (for $\envpol$) and the lower bound of Lemma~\ref{lemma:lowerbound_potential} (for $\envpol'$) to derive the result. Full proof is available in Appendix~\ref{app:lemma:lower_upper_combination}.
\end{proofsketch}

We now show that the algorithm cannot get ``stuck'' making small updates to the same constraints. Specifically, we prove that within a fixed window of $L$ steps, the difference between the probability at $\uncert^*$ and $\uncert^\envpol$ on the critical constraint must be at least halved. The intuition relies on a contradiction: if this difference remained large, Lemma~\ref{lemma:lower_upper_combination} would imply that the value error also remains large. However, we know from Lemma \ref{lemma:policy_iteration_exponential_upperbound} that the value error contracts exponentially. Therefore, this difference must decrease significantly to match this contraction.

\begin{lemma} \phantomsection\label{lemma:subaction_not_repeat}
    Let $\envpol^{0}, \ldots, \envpol^{T}$ be the sequence of policies generated by the policy iteration algorithm.  
    Define the step threshold $L := \log_{\discountfactor} \left( \frac{1 - \discountfactor}{2n^2} \right)$.
    For a specific policy $\envpol^{t}$, let $(s,s',s'') = \argmax_{s,s',s'' \in \States} f_{\envpol^{t}}(s,s',s'')$ be the triple with the maximum potential in $\envpol^{t}$. Then, in any subsequent policy $\envpol^{l}$ with $l > t + L$, we have:
    \begin{equation*}
        \min\!\left( \uncert^*_{s,s'} - \uncert^{\envpol^{l}}_{s,s'}, \;
                     \uncert^{\envpol^{l}}_{s,s''} - \uncert^*_{s,s''} \right)
        \leq \frac{1}{2} 
        \min\!\left( \uncert^*_{s,s'} - \uncert^{\envpol^{t}}_{s,s'}, \;
                     \uncert^{\envpol^{t}}_{s,s''} - \uncert^*_{s,s''} \right).
    \end{equation*}
\end{lemma}

Lemma~\ref{lemma:subaction_not_repeat} (proved in Appendix~\ref{app:lemma:subaction_not_repeat}) shows that after every $L$ iterations of \texttt{RMC-PI}, the amount of $\min(\uncert^*_{s,s'} - \uncert^{\envpol^{l}}_{s,s'},\uncert^{\envpol^{l}}_{s,s''} - \uncert^*_{s,s''})$ reduces by at least half. This means that the {\em most significant bit (MSB)} in its binary presentation changes by at least one index. Our goal is to show that there are polynomially many possible MSBs for this value, hence it can only change a polynomial amount of times. To this end, we introduce the following combinatorial lemma (proved in Appendix~\ref{app:lemma:combinatorial_lemma}):%, which we use later on to prove that the \texttt{RMC-PI} algorithms terminate in polynomial number of iterations.

\begin{lemma} \phantomsection\label{lemma:combinatorial_lemma}
    Let $c$ be a constant positive integer. Let $X$ be a finite set of nonnegative real numbers. Define the set of all signed subset sums as: $$A(X) = \left\{ \left| \sum_{x \in X} g(x) \cdot x \right| \;\middle|\; g : X \to \{-c,\cdots, 0, \cdots, c\} \right\}.$$
    Define the \emph{degree} of a set of positive real numbers $Y \subseteq \mathbb{R}_{>0}$ as the number of distinct MSBs in $Y$:
    $\mathrm{Deg}(Y) = \left| \left\{ \lfloor \log_2 y \rfloor \;\middle|\; y \in Y \right\} \right|.$
    Then for all $X$, it holds that
    $$
        \mathrm{Deg}(A(X) \setminus \{0\}) \in \mathcal{O}(|X| \log |X|).
    $$
\end{lemma}

Lemma~\ref{lemma:combinatorial_lemma} shows that for a set $X$ of non-negative reals, the number of MSBs in the set $A(X)$ is a polynomial bounded by a polynomial in the size of $X$. We derive the final complexity bound by combining the geometric contraction of the discrepancy with the finite number of possible discrepancy values.

\begin{theorem} \phantomsection\label{thm:rmc:iterations}
    \texttt{RMC-PI} terminates with an optimal policy in
    $
        \mathcal{O}\left(n^4 \log n \cdot \frac{\log\left(\frac{1 - \discountfactor}{n^2}\right)}{\log{\discountfactor}} \right)
    $
    iterations where $n= |\States|$.
\end{theorem}

\begin{proofsketch}
The proof is done in three steps:
\begin{compactitem}
    \item \textbf{Geometric Contraction:} Lemma~\ref{lemma:subaction_not_repeat} shows that for the maximizing triple, the discrepancy between the current and optimal transition function decreases by a factor of at least two every $L$ steps.
    \item \textbf{Combinatorial Finiteness:} Lemma~\ref{lemma:combinatorial_lemma} restricts the possible values of the transition probabilities. These values are generated by a finite set of parameters (the nominal probabilities and uncertainty radii). Specifically, by invoking Lemma~\ref{lemma:combinatorial_lemma} on a set with $O(n)$ elements, we show that there are $O(n \log n)$ possible MSBs for the discrepancy. 
    %Consequently, the discrepancy can take only a bounded number of distinct logarithmic scales.
    \item There are at most $n^3$ possibilities for the triple with the highest discrepancy. Hence, after every $O(n^3 \cdot L)$ iterations, the MSB of the highest discrepancy decreases by at least one. Given that there are $O(n \log n)$ possible different MSBs, the algorithm must terminate within $O(n^4 \log n \cdot L)$ iterations. Dropping constants gives the bound in the theorem statement.
\end{compactitem}
Full proof is available in Appendix~\ref{app:thm:rmc:iterations}.
\end{proofsketch}

Since the number of iterations required by the \texttt{RMC-PI} algorithm (Algorithm \ref{algo:RMCPI}) is polynomial, we can conclude that \texttt{RMC-PI} is a strongly polynomial algorithm.

\begin{corollary} \phantomsection\label{corollary:rmc-strong-poly}
\texttt{RMC-PI} is strongly polynomial given an RMC with $n$ states and constant discount factor $\discountfactor$. %Specifically, its runtime is $\mathcal{O}\left(n^6 \log^2 n \cdot \frac{\log\left(\frac{1 - \discountfactor}{n^2}\right)}{\log{\discountfactor}} \right).$
\end{corollary}
% \begin{proof}
%     Theorem~\ref{thm:rmc:iterations} shows that the algorithm runs for $O(n^4 \log n \cdot L)$ iterations. In each iteration, it calls the homotopy method of Algorithm~\ref{algo:homotopy} at most $n$ times, which requires $O(n^2 \log n)$ arithmetic operations. Hence, the total runtime complexity of \texttt{RMC-PI} is $\mathcal{O}\left(n^6 \log^2 n \cdot L \right)$.
% \end{proof}

%% file: rmdp.tex
\section{Robust Markov Decision Processes} \label{sec:rmdp}

In this section, we analyze \texttt{RMDP-PI} (Algorithm~\ref{algo:RMDPPI}) and its convergence to the optimal policy. Let $\rmdp = (\states,\cost,\actions,\succ,\uncert)$ be the input RMDP where $|\states|=n$, $|\actions|=m$ and $\uncert$ is an $\linf$ uncertainty set. We follow a pipeline similar to our analysis of \texttt{RMC-PI} in Section~\ref{sec:rmc}: We first recall the Bellman operator and its properties (Lemma~\ref{lemma:contractive_value_iteration_rmdp}), which imply the exponential convergence rate of \texttt{RMDP-PI} (Lemma~\ref{lemma:consecutive_policy_iteration_bound_rmdp}). We then define a potential function and show a lower and an upper bound on the value error of policies based on their potentials (Lemma~\ref{lemma:lowerbound_potential_rmdp}). Finally, we use the obtained bounds to show that \texttt{RMDP-PI} terminates after a polynomial number of iterations (Theorem~\ref{thm:rmdp:iterations}). 
% the properties of RMDPs and the convergence of \texttt{RMDP-PI} (Algorithm~\ref{algo:RMDPPI}). 

The Bellman optimality operator $\Bellman$ for $\rmdp$ is defined as:
\[
    (\Bellman \valvector)_s = \costvector_s + \discountfactor \; \min_{a \in \actions} \; \max_{\pvector \in \uncertaintyset_{s,a}} {\pvector^\top \valvector}.
\]

We first establish that $\Bellman$ is a contraction mapping with a unique fixed point. This result is standard in the literature of robust dynamic programming (\cite{iyengar2005robust}). The full proof of Lemma \ref{lemma:contractive_value_iteration_rmdp} is available in Appendix \ref{app:lemma:contractive_value_iteration_rmdp}.

\begin{lemma} \label{lemma:contractive_value_iteration_rmdp} 
    The following statements hold:
    \begin{compactitem}
        \item For all $\vvector, \uvector \in \R^n$, $\infnorm{\Bellman \uvector - \Bellman \vvector} \leq \discountfactor \, \infnorm{\uvector - \vvector}$.
        \item For every agent policy $\agentpol$ with value vector $\valvector^\agentpol$, $\Bellman \valvector^\agentpol \preccurlyeq \valvector^\agentpol$.
        \item There exists a unique $\valopt\in \R^n$ such that $\Bellman\valopt= \valopt$.
    \end{compactitem}
\end{lemma}

Next, we prove that the sequence of value functions generated by Policy Iteration is non-increasing and exponentially converging to the optimal value. This monotonicity ensures that the algorithm consistently improves the policy. The full proof of Lemma \ref{lemma:consecutive_policy_iteration_bound_rmdp} is available in Appendix \ref{app:lemma:consecutive_policy_iteration_bound_rmdp}.

\begin{lemma}\label{lemma:consecutive_policy_iteration_bound_rmdp}
    Let $\agentpol^{t}$ and $\agentpol^{t+1}$ be two consecutive policies in the sequence generated by \texttt{RMDP-PI} (Algorithm~\ref{algo:RMDPPI}) and let $\val^t$ be the value of $\agentpol^t$. Then the following statements hold for each $t$:
    \begin{compactitem}
        \item $\valvector^{t+1} \preccurlyeq  \valvector^{t}$
        % \item $\valvector^{t+1} \preccurlyeq \Bellman \valvector^t$
        \item $\infnorm{\valvector^t - \valvector^*} \leq \gamma^t \infnorm{\valvector^0 - \valvector^*}$
    \end{compactitem} 
\end{lemma}

We proceed by defining the potential function. Intuitively, the potential of an action $a$ in state $s$ tracks the additional cost incurred by deviating from the optimal policy $\agentpolopt$ and choosing $a$ at $s$ instead of $\agentpolopt(s)$. 

\begin{definition}[\textbf{Potential function}] \label{def:potential_function_rmdp}
    Given a state $s\in\States$ and an action $a \in \Actions$ the potential function $f(s,a)$ is defined as follows:
    $$
    f(s,a) := \max_{\pvector \in \uncertaintyset_{s,a}}{\pvector^\top \valvector^*} - \max_{\pvector \in \uncertaintyset_{s,\agentpol^*(s)}}{\pvector^\top \valvector^*}
    $$
\end{definition}

The function $f(s,a)$ is the cost-minimization analogue of the standard advantage function in MDPs, taken with respect to the optimal policy $\agentpolopt$.

Next, similar to Lemma~\ref{lemma:lowerbound_potential} and Lemma~\ref{lemma:upperbound_potential}, we establish a lower-bound and an upper-bound on the difference of a policy value vector and the optimal value in terms of the potential function. The full proof of Lemma \ref{lemma:lowerbound_potential_rmdp} is available in Appendix \ref{app:lemma:lowerbound_potential_rmdp}.

\begin{lemma} \phantomsection\label{lemma:lowerbound_potential_rmdp} 
    Let $\agentpol$ be an arbitrary policy. The following holds:
    \begin{compactitem}
        \item \textbf{(Lower-bound)} If $\agentpol(s) = a$ for some state $s \in \States$, then $\valvector^{\agentpol}_s - \valvector^{*}_s \geq \discountfactor f(s,a)$.
        \item \textbf{(Upper-bound)} Suppose $\hat{s} := \argmax_{s \in \States} f(s,\agentpol(s))$. Then, $\infnorm{\valvector^\agentpol - \valvector^*} \leq \frac{\discountfactor}{1 - \discountfactor} f(\hat{s},\agentpol(\hat{s}))$.
    \end{compactitem}
\end{lemma}

We now combine the lower and upper bounds to derive a relationship between the value errors of two different policies. Specifically, we show that if the action chosen at the state with maximum potential remains unchanged between two policies, the value error cannot decrease by more than a factor of $1-\gamma$.

% \begin{lemma} \label{lemma:lower_upper_combination_rmdp}
% Let $\agentpol'$ and $\agentpol$ be two arbitrary policies such that $\valvector^{\agentpol'} \geq \valvector^{\agentpol}$. Let $\hat{s} = \argmax_{s \in \States}{f(s,\agentpol'(s))}$. Also assume that $\agentpol(\hat{s}) = \agentpol'(\hat{s})$. Then we have that:
% $$
%     \infnorm{\valvector^\agentpol - \valvector^*} \geq (1- \discountfactor) \cdot \infnorm{\valvector^{\agentpol'} - \valvector^*}
% $$
% \end{lemma}

\begin{lemma} \label{lemma:lower_upper_combination_rmdp}
Let $\agentpol$ and $\agentpol'$ be two arbitrary policies. Let $\hat{s} = \argmax_{s \in \States}{f(s,\agentpol(s))}$. Also assume that $\agentpol'(\hat{s}) = \agentpol(\hat{s})$. Then we have that:
$$
    \infnorm{\valvector^{\agentpol'} - \valvector^*} \geq (1- \discountfactor) \cdot \infnorm{\valvector^{\agentpol} - \valvector^*}
$$
\end{lemma}

\begin{proof}
    $\infnorm{\valvector^{\agentpol'} - \valvector^*}  \geq (\valvector^{\agentpol'}_{\hat{s}} - \valvector^*_{\hat{s}})
         \geq \discountfactor f(\hat{s},\agentpol'(\hat{s})) %&& (\text{Based on Lemma \ref{lemma:lowerbound_potential_rmdp}}) \\
        = \discountfactor f(\hat{s},\agentpol(\hat{s}))% && (\text{Since } \agentpol'(\hat{s}) = \agentpol(\hat{s})) \\
        \geq (1 - \discountfactor) \infnorm{\valvector^{\agentpol} - \valvector^*}% && (\text{Based on Lemma \ref{lemma:upperbound_potential_rmdp}})
    $\\
    where the second and the last inequalities are due to the bounds in Lemma~\ref{lemma:lowerbound_potential_rmdp}.
\end{proof}

% \begin{proof}
%     \begin{align*}
%         \infnorm{\valvector^\agentpol - \valvector^*} & \geq (\valvector^\agentpol_{\hat{s}} - \valvector^*_{\hat{s}}) && (\text{By definition of } \ell_\infty \text{ norm}) \\
%         & \geq \discountfactor f(\hat{s},\agentpol(\hat{s})) && (\text{Based on Lemma \ref{lemma:lowerbound_potential_rmdp}}) \\
%         & = \discountfactor f(\hat{s},\agentpol'(\hat{s})) && (\text{Since } \agentpol(\hat{s}) = \agentpol'(\hat{s})) \\
%         & \geq (1 - \discountfactor) \infnorm{\valvector^{\agentpol'} - \valvector^*} && (\text{Based on Lemma \ref{lemma:upperbound_potential_rmdp}})
%     \end{align*}
%     This completes the proof.
% \end{proof}

We now utilize the exponential convergence rate established in Lemma \ref{lemma:consecutive_policy_iteration_bound_rmdp} to derive that an action with maximum potential at some iteration cannot persist for more than $L$ iterations. The full proof of Lemma \ref{lemma:action_elimination} is available in Appendix \ref{app:lemma:action_elimination}.

\begin{lemma} \phantomsection\label{lemma:action_elimination}
    Let $\agentpol^{0}, \ldots, \agentpol^{T}$ be the sequence of agent policies generated by \texttt{RMDP-PI}. Let $L = \log_{\discountfactor} \left( 1  -\discountfactor \right)$ and for each policy $\agentpol^l$, let $\hat{s} = \argmax_{s} f(s,\agentpol^l(s))$. Then for all $k > l + L$, it holds that $\agentpol^k(\hat{s}) \neq \agentpol^l(\hat{s})$.
\end{lemma}

Finally, we aggregate the elimination steps to provide a polynomial upper bound on the total number of iterations. Since an action is eliminated from the set of potential maximizers every $L$ steps, the algorithm must terminate within $O(n \cdot m \cdot L)$ iterations for $n=|\states|$ and $m=|\actions|$. The full proof of Theorem \ref{thm:rmdp:iterations} is available in Appendix \ref{app:thm:rmdp:iterations}.

\begin{theorem} \phantomsection\label{thm:rmdp:iterations}
    \texttt{RMDP-PI} terminates with an optimal policy in $\mathcal{O}\left(n \cdot m \cdot \frac{\log{(1-\discountfactor)}}{\log{\discountfactor}}\right)$ iterations.
\end{theorem}

The policy evaluation step of \texttt{RMDP-PI} is strongly polynomial due to Corollary~\ref{corollary:rmc-strong-poly}. The policy improvement is also trivially strongly polynomial, implying the following corollary immediately:

\begin{corollary}
    \texttt{RMDP-PI} is strongly polynomial given an RMDP $\rmdp$ with $n$ states, $m$ actions and constant discount factor $\discountfactor$.
\end{corollary}

% \begin{remark}
%     The complexity result of Theorem~\ref{thm:rmdp:iterations} also applies to interval MDPs, a similar notion studied in the formal methods literature. RMDPs with $L_\infty$ uncertainty sets can be viewed as a special case of interval MDPs. Since the proof follows similarly with minor changes, details are provided in Appendix~\ref{app:interval}.
% \end{remark}

%% file: conclu.tex
\section{Conclusion}
In this paper, we studied discounted $(s,a)$-rectangular RMDPs with $\linf$ uncertainty sets, a model that captures classical MDPs and turn-based stochastic games. Our main contribution is to resolve a fundamental algorithmic open problem for this setting. We showed that a robust policy iteration algorithm terminates in strongly polynomial time when the discount factor is fixed. Several directions remain open and relevant. While the present work establishes the theoretical
foundations, it is an interesting future direction to complement the worst-case analysis with empirical studies of robust policy iteration on data-driven uncertainty sets, and to explore whether the potential-based arguments suggest improved practical variants or stopping criteria. Another related major question is whether policy iteration for MDPs admits polynomial or strongly polynomial bounds. \cite{DBLP:conf/icalp/Fearnley10} showed that one variant can take exponentially many iterations for discounted-sum MDPs, leaving open whether other variants run in polynomial time.

%% file: appendix.tex
\section{Reduction from Turn Based Stochastic Games to $\linf$ RMDPs} \label{app:stochastic-games}
A turn based stochastic game $G$ is defined as a tuple $(\states_1,\states_2,\states_r,\cost, \succ, \trans)$, where 
\begin{compactitem}
    \item $\states_1, \states_2$ and $\states_r$ are disjoint sets. $\states_1$ is a set of player-1 states, $\states_2$ is a set of player-2 states, and $\states_r$ is a set of randomized states. Let $\states = \states_1 \uplus \states_2 \uplus \states_r$
    \item $\cost \colon \states \to \R$ is a cost function
    % \item $\actions_1$ is a finite set of player-1 actions, $\actions_2$ is a finite set of player-2 actions 
    \item $\succ \colon \states_1 \cup \states_2 \to 2^\states$ specifies the set of possible successors from each $\states_1$ and $\states_2$ state. 
    \item $\trans \colon \states_r \to \Delta(\states)$ is the randomized transition function from $\states_r$ states. 
\end{compactitem}
The semantics are defined by player strategies as follows: a randomized strategy $\agentpol_i$ of player-$i$ is a function $\agentpol_i \colon \states_i \to \Delta(\states)$ where $\textit{support}(\pi_i(s)) \subseteq \succ(s)$. A run of $G$ is then an infinite sequence $\pi_0, \pi_1 \dots $ where if $\pi_j \in \states_i$ for some $i \in \{1,2\}$, then $\Pr[\pi_{j+1}|\pi_j] = \agentpol_i(\pi_{j})[\pi_{j+1}]$, and if $\pi_j \in \states_r$, then $\Pr[\pi_{j+1}|\pi_j] = \trans(\pi_j)[\pi_{j+1}]$. The cylinder construction of \cite{DBLP:books/daglib/0020348} synthesizes the full probability distribution over infinite sequences of states in $G$. For a discount factor $\discountfactor$, the discounted sum of costs over a run $\pi$ is then defined as
\[
c(\pi) = \sum_{j=0}^\infty \discountfactor^j \cost(\pi_j)
\]
The goal of player-1 is to minimize the expected cost of the induced run while player-2 wants to maximize it. Hence the value of the game is defined as 
\[
\val_G(s) = \min_{\agentpol_1} \max_{\agentpol_2} \E_{\pi\sim(\agentpol_1,\agentpol_2)}[c(\pi)]
\]

We show that any game $G=(\states^G_1,\states^G_2,\states^G_r,\cost^G, \succ^G, \trans^G)$ is equivalent to an RMDP $\rmdp=(\states^\rmdp,\cost^\rmdp,\actions^\rmdp,\succ^\rmdp,\uncert^\rmdp)$ where $\uncert^\rmdp$ is an $\linf$ uncertainty set specified by the nominal transition probabilities $\nominal^\rmdp(s,a)$ and the uncertainty radii $\radius(s,a)$ for each state-action pair $(s,a)$. The RMDP $\rmdp$ is constructed as follows:
\begin{compactitem}
    \item $\states^\rmdp = \states^G_1 \cup \states^G_2 \cup \states^G_r$
    \item $\cost^\rmdp = \cost^G$
    \item For each $s \in \states^G_1$ and $s' \in \succ^G(s)$, there is an action $a \in \actions^\rmdp$ where $\succ^\rmdp(s,a)=\{s'\}$, $\nominal^\rmdp(s,a)[s']=1$ and $\radius(s,a)=0$
    \item For each $s \in \states^G_r$, there is an action $a \in \actions^\rmdp$ where $\succ^\rmdp(s,a)=\textit{support}(\trans^G(s))$, $\nominal^\rmdp(s,a)[s']=\trans^G(s)[s']$ and $\radius^\rmdp(s,a)=0$. 
    \item For each $s \in \states^G_2$, there is an action $a \in \actions^\rmdp$ where $\succ^\rmdp(s,a)=\succ^G(s)$, $\nominal^\rmdp(s,a)$ is an arbitrary distribution in $\Delta(\succ^G(s))$, and $\radius(s,a)=1$. 
\end{compactitem}
We show that every strategy profile in $G$ has an equivalent policy profile in $\rmdp$ and vice versa. 

\begin{compactenum}
    \item Let $(\agentpol_1,\agentpol_2)$ be a strategy profile in $G$. We construct a policy profile $(\agentpol,\envpol)$ in $\rmdp$ that generates the same probability distribution over the set of infinite sequences of states in $\states^\rmdp$ simply because $\Pr_{(\agentpol_1,\agentpol_2)}(s)[s'] = \Pr_{(\agentpol,\envpol)}[s](s')$. Specifically, $\agentpol(s) = \agentpol_1(s)$ for all $s \in \states_1 \cup \states_r$. Note that for $s \in \states_2$, $|\actions^\rmdp(s)|=1$, hence $\agentpol(s)$ is defined in the trivial manner. For $\envpol$, the environment has a single choice in all $s \in \states_1 \cup \states_r$. For $s \in \states_2$, we let $\envpol(s)=\agentpol_2(s)$. This is specifically a valid choice for the environment because $\uncert^\rmdp(s) = \Delta(\succ(s))$. By construction, $\Pr_{(\agentpol_1,\agentpol_2)}(s)[s'] = \Pr_{(\agentpol,\envpol)}[s](s')$ for all states $s,s' \in \states^\rmdp$. 
    \item Let $(\agentpol,\envpol)$ be a policy profile in $\rmdp$. We construct $(\agentpol_1,\agentpol_2)$ to be a strategy profile in $G$ which induces the same transition function as $(\agentpol,\envpol)$. Given the construction of $\rmdp$, it is possible to put $\agentpol_1(s) = \agentpol(s)$ for all $s \in \states_1$. For $\agentpol_2$, let $s \in \states_2$ be arbitrary. We define $\agentpol_2(s)=\envpol(s,a)$ for the unique action $a$ defined at $s$ in $\actions^\rmdp$. Again, by construction, $\Pr_{(\agentpol_1,\agentpol_2)}(s)[s'] = \Pr_{(\agentpol,\envpol)}[s](s')$ for all states $s,s' \in \states^\rmdp$.
\end{compactenum}

\section{Proof of Lemma \ref{lemma:contractive_value_iteration}} \label{app:lemma:bellman}
\begin{proof}
    \begin{compactitem}
        \item  Fix $s \in \States$. Without loss of generality, we assume $(\Bellman \uvector)_s \geq (\Bellman \vvector)_s$. Let $\pvector'$ and $\pvector''$ denote the maximizers in $\uncertaintyset(s)$ for $\uvector$ and $\valvector$, respectively. We write the difference as follows:
    \begin{align*}
        (\Bellman\uvector)_s - (\Bellman\vvector)_s
        &= \discountfactor (\pvector' \uvector - \pvector'' \vvector) && \text{(Definition of $\Bellman$)} \\
        &\leq \discountfactor (\pvector' \uvector - \pvector' \vvector) && \text{(Optimality of $\pvector''$)} \\
        &= \discountfactor \pvector' (\uvector - \vvector) && \text{(Linearity)} \\
        &\leq \discountfactor \infnorm{\uvector - \vvector}. && \text{($\pvector'$ is a distribution)}
    \end{align*}
    Since this bound holds for an arbitrary state $s$, the result follows.
    \item For any state $s \in \States$, we have:
    \begin{align*}
        (\Bellman \valvector^{\envpol})_s 
        & = \costvector_s + \discountfactor \max_{\pvector \in \uncertaintyset(s)} \pvector^\top \valvector^{\envpol} && \text{(Definition of $\Bellman$)} \\
        & \geq \costvector_s + \discountfactor \uncert^{\envpol}_s \cdot \valvector^{\envpol} && \text{(Since $\uncert^{\envpol}_s \in \uncertaintyset(s)$)} \\
        & = \valvector^{\envpol} && \text{Definition of $\valvector^\envpol$}
    \end{align*}
    Since this inequality holds for all states, the result follows.
    \item Direct result of the Banach fixed-point theorem.
    \end{compactitem}
\end{proof}

\section{Proof of Lemma \ref{lemma:consecutive_policy_iteration_bound}} \label{app:lemma:convergence}
\begin{proof}
    \begin{compactitem}
        \item Let $\uncert^t$ denote the transition matrix corresponding to the policy in iteration $t$. We analyze the difference between the value vectors:
    \begin{align*}
        \valvector^{t+1} - \valvector^{t} 
        & = (\identitymatrix - \gamma \uncert^{t+1})^{-1} \costvector - \valvector^t && \text{(By definition of value)} \\
        & = (\identitymatrix - \gamma \uncert^{t+1})^{-1} (\identitymatrix - \gamma \uncert^{t+1}) \left[ (\identitymatrix - \gamma \uncert^{t+1})^{-1} \costvector - \valvector^t \right] && \text{(Multiplying by $\identitymatrix$)} \\
        & = (\identitymatrix - \gamma \uncert^{t+1})^{-1} \left[ \costvector - (\identitymatrix - \gamma \uncert^{t+1}) \valvector^t \right] && \text{(Rearranging)} \\
        & = (\identitymatrix - \gamma \uncert^{t+1})^{-1} \left[ (\costvector + \gamma \uncert^{t+1} \valvector^t) - \valvector^t \right] && \text{(Expanding inner terms)} \\
        & = (\identitymatrix - \gamma \uncert^{t+1})^{-1} \left[ \Bellman \valvector^t - \valvector^t \right] && \text{(Definition of $\Bellman$ and $\envpol^{t+1}$)}
    \end{align*}
    By Lemma \ref{lemma:contractive_value_iteration}, $\Bellman \valvector^t - \valvector^t \succcurlyeq \mathbf{0}$. We express the inverse matrix using the Neumann series:
    \[
        (\I - \discountfactor \uncert^{t+1})^{-1} = \sum_{i=0}^\infty (\discountfactor \uncert^{t+1})^i.
    \]
    Since $\discountfactor$ and all entries of $\uncert^{t+1}$ are non-negative, the sum consists of non-negative terms. Consequently, $(\I - \discountfactor \uncert^{t+1})^{-1}$ is component-wise non-negative. This implies $\valvector^{t+1} - \valvector^t \succcurlyeq \mathbf{0}$.
    \item First note that it can be easily shown that $\valvector^{t+1} \succcurlyeq \Bellman \valvector^{t}$ as follows:
    \begin{align*}
        \valvector^{t+1} - \Bellman \valvector^t 
        & = \valvector^{t+1} - (\costvector + \gamma \uncert^{t+1} \valvector^t) && \text{(Definition of $\Bellman$ and $\envpol^{t+1}$)} \\
        & = \valvector^{t+1} - \left( (\identitymatrix - \gamma \uncert^{t+1})\valvector^{t+1} + \gamma \uncert^{t+1} \valvector^t \right) && \text{(Replacing $\costvector$ using Equation \ref{Eq:RMCValue})} \\
        & = \valvector^{t+1} - (\valvector^{t+1} - \gamma \uncert^{t+1} \valvector^{t+1} + \gamma \uncert^{t+1} \valvector^t) && \text{(Expanding)} \\
        & = \gamma \uncert^{t+1} (\valvector^{t+1} - \valvector^t) && \text{(Simplifying)}
    \end{align*}
    Since $\discountfactor > 0$, $\uncert^{t+1} \geq 0$ entrywise, and by the first bullet point $\valvector^{t+1} - \valvector^t \succcurlyeq \mathbf{0}$, we conclude $\valvector^{t+1} \succcurlyeq \Bellman \valvector^{t}$.

    We now proceed by proving the second statement of the Lemma:

    We use induction on $t$. The base case $t=0$ is trivial. For $t > 0$:
    \begin{align*}
        \infnorm{\valvector^t - \valvector^*} 
        & \leq \infnorm{\Bellman \valvector^{t-1} - \valvector^*} && \text{(Since $\valvector^* \succcurlyeq \valvector^t \succcurlyeq \Bellman \valvector^{t-1}$ as above)} \\
        & = \infnorm{\Bellman \valvector^{t-1} - \Bellman \valvector^*} && \text{(Fixed point property by Lemma \ref{lemma:contractive_value_iteration})} \\
        & \leq \gamma \infnorm{\valvector^{t-1} - \valvector^*} && \text{(Contraction of $\Bellman$ by Lemma \ref{lemma:contractive_value_iteration})} \\
        & \leq \gamma^t \infnorm{\valvector^0 - \valvector^*} && \text{(By induction)}
    \end{align*}
    \end{compactitem}
\end{proof}

\section{Proof of Lemma \ref{lemma:lowerbound_potential}} \label{app:lemma:lowerbound_potential}
\begin{proof}
    \begin{align*}
        \valvector^*_s - \valvector^\envpol_s 
        &= \left( \costvector + \discountfactor \uncert^* \valvector^* \right)_s 
         - \left( \costvector + \gamma \uncert^\envpol \valvector^\envpol \right)_s 
        \\
        &\geq \left( \costvector + \discountfactor \uncert^* \valvector^* \right)_s 
         - \left( \costvector + \gamma \uncert^\envpol \valvector^* \right)_s 
        && \text{(Since $\valvector^* \succcurlyeq \valvector^\envpol$)} \\
        &= \gamma \left( \uncert^*_{s} - \uncert^\envpol_{s} \right) \valvector^*
        && \text{(Simplifying)}
    \end{align*}
    Let $\epsilon := \min(\uncert^*_{s,s'} - \uncert^\envpol_{s,s'}, \uncert^\envpol_{s,s''} - \uncert^*_{s,s''})$. By the symmetry of $f$ under $s',s''$ and the fact that $\epsilon \leq 0$ cannot produce the argmax, we can assume without loss of generality that $\epsilon \geq 0$. Now if $f_\envpol(s,s',s'') \leq 0$ the bound is trivial since $\valvector^*_s \geq \valvector^\envpol_s$. Now, consider perturbing $\uncert^\envpol_s$ by increasing $\uncert^\envpol_{s,s'}$ by $\epsilon$ and decreasing $\uncert^\envpol_{s,s''}$ by $\epsilon$. Let $\uncert^{\envpol'}_s$ denote the resulting vector. By construction $0 \le \epsilon \le \uncert^*_{s,s'} - \uncert^\tau_{s,s'}$, so $\uncert^\tau_{s,s'} + \epsilon \le \uncert^*_{s,s'} \le \hat{P}_{s,s'} + \delta_s$. Similarly $\uncert^\tau_{s,s''} - \epsilon \ge \uncert^*_{s,s''} \ge \hat{P}_{s,s''} - \delta_s$. The row sum is preserved. The perturbation satisfies:
    \[
         \uncert^{\envpol'}_s \valvector^* - \uncert^\envpol_s \valvector^* = \epsilon \left( \valvector^*_{s'} - \valvector^*_{s''} \right) = f_\envpol(s,s',s'') .
    \]
    We now derive the bound using the optimality of $\uncert^*$:
    \begin{align*}
        & \uncert^*_s \valvector^* \geq \uncert^{\envpol'}_s \valvector^* && \text{(Optimality of $\uncert^*$)} \\
        \implies \quad & \uncert^*_s \valvector^* - \uncert^\envpol_s \valvector^* \geq \uncert^{\envpol'}_s \valvector^* - \uncert^\envpol_s \valvector^* && \text{(Subtracting $\uncert^\envpol_s \valvector^*$)} \\
        \implies \quad & \left( \uncert^*_s - \uncert^\envpol_s \right) \valvector^* \geq \epsilon \left( \valvector^*_{s'} - \valvector^*_{s''} \right) 
        && \text{(Substitution)} \\
        \implies \quad & \left( \uncert^*_s - \uncert^\envpol_s \right) \valvector^* \geq f_\envpol(s,s',s'') 
        && \text{(Definition of $f_\envpol$)}
    \end{align*}
    Combining this result with the first inequality in the proof, we obtain:
    \[
        \valvector^*_s - \valvector^\envpol_s \geq \gamma f_\envpol(s,s',s'').
    \]
\end{proof}

\section{Proof of Lemma \ref{lemma:upperbound_potential}} \label{app:lemma:upperbound_potential}

\begin{proof}
    First, we relate the value difference to the transition difference using the Bellman equation.
    \begin{align*}
        \valvector^* - \valvector^\envpol
        &= (\identitymatrix - \discountfactor\uncert^*)^{-1} \costvector - (\identitymatrix - \discountfactor\uncert^\envpol)^{-1} \costvector 
        && \text{(By Equation~\ref{Eq:RMCValue})} \\
        &= (\identitymatrix - \discountfactor\uncert^\envpol)^{-1} \left[ (\identitymatrix - \discountfactor\uncert^\envpol) - (\identitymatrix - \discountfactor\uncert^*) \right] (\identitymatrix - \discountfactor\uncert^*)^{-1} \costvector
        && \text{(Using $A^{-1} - B^{-1} = B^{-1}(B - A)A^{-1}$)} \\
        &= (\identitymatrix - \discountfactor\uncert^\envpol)^{-1} \left[ \discountfactor (\uncert^* - \uncert^\envpol) \right] \valvector^*
        && \text{(Simplifying terms)}
    \end{align*}
    We now take the infinity norm of both sides:
    \begin{align*}
        \infnorm{\valvector^* - \valvector^\envpol} 
        &= \discountfactor \infnorm{(\identitymatrix - \discountfactor\uncert^\envpol)^{-1} (\uncert^* - \uncert^\envpol) \valvector^*} \\
        &\leq \discountfactor \infnorm{(\identitymatrix - \discountfactor\uncert^\envpol)^{-1}} \infnorm{(\uncert^* - \uncert^\envpol) \valvector^*} 
        && \text{(Submultiplicativity of $\linf$ norm)}
    \end{align*}
    We analyze the term $\infnorm{(\identitymatrix - \discountfactor\uncert^\envpol)^{-1}}$ using the Neumann series. Since $\uncert^\envpol$ is a transition matrix, it is row-stochastic (all rows sum to 1 and entries are non-negative). Consequently, its induced infinity norm is $\infnorm{\uncert^\envpol} = 1$. Furthermore, the product of stochastic matrices is stochastic, implying $\infnorm{(\uncert^\envpol)^t} = 1$ for all $t \geq 0$.
    \begin{align*}
        \infnorm{(\identitymatrix - \discountfactor\uncert^\envpol)^{-1}} 
        &= \infnorm{\sum_{t=0}^\infty (\discountfactor \uncert^\envpol)^t} 
        && \text{(Neumann Series expansion)} \\
        &\leq \sum_{t=0}^\infty \discountfactor^t \infnorm{(\uncert^\envpol)^t} 
        && \text{(Triangle inequality)} \\
        &= \sum_{t=0}^\infty \discountfactor^t \cdot 1
        && \text{(Since $(\uncert^\envpol)^t$ is stochastic)} \\
        &= \frac{1}{1 - \discountfactor} 
        && \text{(Geometric series sum)}
    \end{align*}
    Substituting this back into the previous inequality:
    \[
        \infnorm{\valvector^* - \valvector^\envpol} \leq \frac{\discountfactor}{1 - \discountfactor} \infnorm{(\uncert^* - \uncert^\envpol) \valvector^*}.
    \]

    It remains to bound the term $\infnorm{(\uncert^* - \uncert^\envpol) \valvector^*}$. We do this by interpreting the difference matrix as a collection of mass transfers.
    
    \noindent\textbf{Mass transfer interpretation.}  
    We express the difference $\uncert^* - \uncert^\envpol$ as a sequence of flows between coordinates. Define a tensor $m \in \mathbb{R}^{n \times n \times n}$ where $m(i,j,k)$ represents the probability mass moved to state $j$ (where $\envpol$ has a deficit) from state $k$ (where $\envpol$ has excess) in row $i$.
    Since both $\uncert^\envpol$ and $\uncert^*$ lie within the same convex $\linf$-ball, the vector connecting them ($\uncert^* - \uncert^\envpol$) can be decomposed into a finite sum of \emph{feasible} elementary redistributions. For $i \in \states$ Suppose $K_i \subseteq \states \times \states$ is such that 
    $\uncert^\envpol_i + \sum\limits_{(j,k) \in K_i} m(i,j,k)(e_j - e_k) = \uncert^*_i$, where $e_j$ is the $j$-th standard basis vector. We may further assume the decomposition satisfies $\valvector^*_j \geq \valvector^*_k$ whenever $m(i,j,k) > 0$: otherwise, by the optimality of $\uncert^*_i$, swapping such a pair back would yield a feasible distribution with strictly larger dot product against $\valvector^*$. This also gives $(\uncert^*_i - \uncert^\envpol_i)\valvector^* \geq 0$ in every row, so the absolute values below can be dropped.

    % \[
    % We choose $m(i,j,k)$ such that:
    %     m(i,j,k) \leq \max\!\left( \min\!\left( \uncert^\envpol_{i,j} - \uncert^*_{i,j}, \; \uncert^*_{i,k} - \uncert^\envpol_{i,k} \right), \; 0 \right).
    % \]
    % This construction ensures that we strictly account for the difference between the two transition matrices.

    \noindent\textbf{Bounding the total change.}  
    We analyze the row $i$ that maximizes the difference:
    \begin{align*}
        \infnorm{(\uncert^* - \uncert^\envpol) \valvector^*} 
        &= \max_{i} (\uncert^*_i - \uncert^\envpol_i) \valvector^*
        && \text{(Optimality of $\uncert^*$)} \\
        &= \max_{i} \sum_{(j,k) \in K_i} m(i,j,k) (\valvector^*_j - \valvector^*_k)
        && \text{(Decomposing into transfers)}
    \end{align*}
    Note that for any specific triplet $(i, j, k)$, the contribution to the sum $m(i,j,k) \left| \valvector^*_j - \valvector^*_k \right|$ is bounded by the potential function $f_\envpol(i, j,k)$ which is upper-bounded by $f_\envpol(s, s', s'')$:
    \begin{align*}
        \max_{i} \sum_{(j,k) \in K_i} m(i,j,k)  (\valvector^*_j - \valvector^*_k)  
        &\leq \max_{i} \sum_{(j,k) \in K_i} f_\envpol(i, j, k) 
        && \text{(Definition of $f_\envpol$)} \\
        &\leq \sum_{(j,k) \in K_i} f_\envpol(s, s', s'')  
        && \text{(Bounding by global max)} \\
        &\leq n^2 f_\envpol(s, s', s'') 
        && \text{(Summing over at most $n^2$ pairs)}
    \end{align*}
    Substituting this back into our earlier bound:
    \[
        \infnorm{\valvector^* - \valvector^\envpol} \leq \frac{\discountfactor}{1 - \discountfactor} \left( n^2 f_\envpol(s, s', s'') \right).
    \]
    This completes the proof.
\end{proof}

\section{Proof of Lemma \ref{lemma:lower_upper_combination}} \label{app:lemma:lower_upper_combination}

\begin{proof}
    We chain the upper bound of Lemma~\ref{lemma:upperbound_potential} (for $\envpol$) and the lower bound of Lemma~\ref{lemma:lowerbound_potential} (for $\envpol'$) to derive the result. First, we apply the upper bound to the policy $\envpol$:
    \begin{align*}
        \infnorm{\valvector^* - \valvector^\envpol} 
        & \leq \frac{n^2 \discountfactor}{1 - \discountfactor} f_\envpol(s,s',s'') 
        && \text{(By Lemma \ref{lemma:upperbound_potential})}\\
        % & = \frac{n^2 \discountfactor}{1 - \discountfactor} \min\left( \uncert^*_{s,s'} - \uncert^\envpol_{s,s'}, \; \uncert^\envpol_{s,s''} - \uncert^*_{s,s''} \right) (\valvector^*_{s'} - \valvector^*_{s''}) 
        % && \text{(By Definition of $f_\envpol$ (\ref{def:potential_function}) )}
        & \leq \frac{n^2 \discountfactor}{1 - \discountfactor} \cdot 2 f_{\envpol'}(s,s',s'') 
        && \text{(By definition of $f_{\envpol'}$ and assumption)}
    \end{align*}

    The last inequality is true because $\valvector^*_{s'} \geq \valvector^*_{s''}$ since we are considering the maximum potential. Finally, we link this to the value error of $\envpol'$ using the lower bound:
    \[
        \infnorm{\valvector^* - \valvector^\envpol} 
        \leq \frac{2 n^2}{1 - \discountfactor} \left( \discountfactor f_{\envpol'}(s,s',s'') \right) \leq \frac{2 n^2}{1 - \discountfactor} \left( \valvector^*_{s} - \valvector^{\envpol'}_{s} \right) \leq \frac{2 n^2}{1 - \discountfactor} \infnorm{\valvector^* - \valvector^{\envpol'}}  \\
    \]
    Where the second inequality comes from Lemma~\ref{lemma:lowerbound_potential} and the third one is due to the definition of $\linf$ norm. Rearranging the terms gives the claimed bound:
    $
        \infnorm{\valvector^* - \valvector^{\envpol'}} \geq \frac{1 - \discountfactor}{2n^2} \infnorm{\valvector^* - \valvector^\envpol}.
    $
\end{proof}

\section{Proof of Lemma \ref{lemma:subaction_not_repeat}} \label{app:lemma:subaction_not_repeat}
\begin{proof}
    Let $s,s',s''$ be the critical triplet defined in the lemma.  
    Assume, for the sake of contradiction, that there exists some step $l > t + L$ such that:
    \begin{equation*}
        \min\!\left( \uncert^*_{s,s'} - \uncert^{\envpol^{l}}_{s,s'}, \;
                     \uncert^{\envpol^{l}}_{s,s''} - \uncert^*_{s,s''} \right)
        > \frac{1}{2} 
        \min\!\left( \uncert^*_{s,s'} - \uncert^{\envpol^{t}}_{s,s'}, \;
                     \uncert^{\envpol^{t}}_{s,s''} - \uncert^*_{s,s''} \right).
    \end{equation*}

    This assumption allows us to invoke Lemma~\ref{lemma:lower_upper_combination}, providing a lower bound on the error at step $l$. Simultaneously, standard Policy Iteration convergence provides an upper bound. We combine these as follows:
    \begin{align*}
        \infnorm{\valvector^{\envpol^{l}} - \valvector^*} 
        &\leq \discountfactor^{\,l - t} \infnorm{\valvector^{\envpol^{t}} - \valvector^*}
        && \text{(By Lemma 
        \ref{lemma:policy_iteration_exponential_upperbound})} \\[0.5em]
        \infnorm{\valvector^{\envpol^{l}} - \valvector^*}
        &\geq \frac{1 - \discountfactor}{2n^2} \infnorm{\valvector^{\envpol^{t}} - \valvector^*}
        && \text{(By Lemma~\ref{lemma:lower_upper_combination} and assumption)}
    \end{align*}
    
    Combining these two inequalities implies:
    \begin{align*}
        \implies \quad
        \discountfactor^{\,l - t} &\geq \frac{1 - \discountfactor}{2n^2} \\[0.5em]
        \implies \quad
        l - t &\leq \log_{\discountfactor} \left( \frac{1 - \discountfactor}{2n^2} \right)
        && \text{(Taking $\log_\discountfactor$, inequality flips since $\discountfactor < 1$)}.
    \end{align*}

    This conclusion ($l - t \leq L$) directly contradicts the premise that $l > t + L$.  
    Thus, the assumption is false, and the difference must have decreased by at least half.
\end{proof}

\section{Proof of Lemma \ref{lemma:combinatorial_lemma}} \label{app:lemma:combinatorial_lemma}
\begin{proof}
    This result follows immediately from Theorem~\ref{thm: main1}. This theorem bounds the number of distinct scales (powers of 2) that can be generated by linear combinations with small integer coefficients. 
\end{proof}

    Let us begin with the definition. 
\begin{definition}
We call a \emph{dyadic interval} any interval either of the form $[2^i,2^{i+1})$ or of the form $[-2^{i+1},-2^i)$ for some integer $i$. $[0,0]$ is treated as a separate dyadic interval. 
\end{definition}
Observe that the collection of dyadic intervals offers a countable partition of $\mathbb{R}$. For $i \in \mathbb{Z}$, we say that a real number has degree $i$ if it belongs to the component $[2^i, 2^{i+1})$ or to the component $[-2^{i+1},-2^i)$.

For a finite subset $X \subseteq \mathbb{R}$ of the real numbers, and a positive real $C$ we denote by $A(X,C)$ the set of real numbers of the form 
$$\sum_{x \in X}f(x) \cdot x,
$$
where $f: X \to \mathbb{Z}$ is any function with $\infnorm{f} \leq C$. For a subset $Y$ of $\mathbb{R}$, we denote by
$$\text{Dyad}(Y),
$$
the collection of dyadic intervals containing elements of $Y$.
\begin{theorem} \phantomsection\label{thm: main1}
Let $C$ be in $\mathbb{R}_{>0}$. Then for all $X$ finite subset of $\mathbb{R}$ we have that 
$$|\textup{Dyad}(A(X,C))|=O_C(|X| \cdot \textup{log}(|X|)).
$$
\end{theorem}
Shortly after and independently, an effective strengthening of Theorem \ref{thm: main1} was discovered and established by a custom mathematics research agent, named \emph{Aletheia}, built upon Gemini Deep Think at Google DeepMind under the lead of Tony Feng, see \cite{Aletheia}.

\emph{Aletheia}'s theorem goes as follows. 
\begin{theorem} \label{thm: main2}
Let $C$ be in $\mathbb{R}_{>0}$. Then for all $X$ finite subset of $\mathbb{R}$ we have that    
$$|\textup{Dyad}(A(X,C))| \le 2 \cdot (2|X|-1)(\lfloor \log_2(2|X|C+1) \rfloor + 2)+1.$$
\end{theorem}
The prompt that the last author of the present paper fed to \emph{Aletheia} was as follows:\\
\textbf{Let $C$ be a positive integer. For 
$X$ a finite set of real numbers, denote with 
$A(X,C)$ the set of real numbers obtained by taking integer linear combinations of elements of $X$, with coefficients bounded by $C$. Find an upper bound, which is polynomial in $|X|$, for the number of intervals of the form  $[2^n,2^{n+1})$, with $n$ nonnegative integer, that intersect an element of $A(X,C)$. If possible find an asymptotically sharp upper bound.}

At the center of both arguments to prove Theorem \ref{thm: main1} and Theorem \ref{thm: main2}, there is an application of Siegel's Lemma. However, this is implemented quite differently in the two proofs.

\paragraph{Overview of proof of Theorem \ref{thm: main1}}
Our idea to prove Theorem \ref{thm: main1} is inspired by a standard move in modern number theory, which is to examine a function field analogue first. In this case the dyadic subdivision corresponds simply to the degree partition of the set of polynomials \footnote{We shall not make the analogy entirely strict by making $X$ to be a subset of the completion of the function field with respect to the degree valuation, as the purpose of the function field analogy is to simply illustrate the proof in an idealized setting.}. The analogue for function fields of Theorem \ref{thm: main1} is the following simple proposition that we shall prove right away. The proof will provide us with a clear road map for the proof of Theorem \ref{thm: main1}. 
\begin{proposition} \label{prop: function fields}
Let $\mathbb{F}$ be a finite field and fix $m$ a nonnegative integer. Let $X$ be a subset of the polynomial ring $\mathbb{F}[T]$. Then one sees at most $(m + 1) \cdot |X|$ degrees as one runs in the set of non-zero polynomials of the form
$$\sum_{p \in X}f(p) \cdot p,
$$
where $\textup{deg}(f(p)) \leq m$ for each $p$ in $X$. 
\begin{proof}
The set described is the $\mathbb{F}$-vector space $V$ generated by the elements in $X \cdot \{1, \ldots, t^{m}\}$. Since it is $\mathbb{F}$-spanned by $(m + 1) \cdot |X|$ elements, this vector space $V$ has $\mathbb{F}$-dimension at most $(m + 1) \cdot |X|$. However, $(m + 1) \cdot |X| + 1$  non-zero polynomials of pairwise distinct degree always form, in particular, a linearly independent set. But then $V$ cannot contain a linearly independent subset exceeding its dimension. Thus, there are no such $(m + 1) \cdot |X| + 1$ non-zero polynomials to be seen in $V$, yielding the desired conclusion. 
\end{proof}
\end{proposition} 
With the proof of Proposition \ref{prop: function fields} in our hands, we can now briefly overview the main ideas of the proof of Theorem \ref{thm: main1}. To mimic the proof, we proceed as follows: \\
$(1)$ In the context of Theorem \ref{thm: main1}, the field $\mathbb{F}$ is replaced by $[-C_1,C_1] \cap \mathbb{Z}$, which does not have any standard algebraic structure. Nevertheless, we can mimic the notion of linear independence of $N$ real numbers: we say that $N$ reals are ``$C_1$-linearly independent" in case they don't have a non-trivial linear dependency using only integers of absolute value going up to $C_1$.  \\
$(2)$ We now observe that for any positive real $C_1$, there exists a positive real $D$ such that if we have any collection of non-zero real numbers whose degrees are pairwise spaced by $D$, then
the collection is $C_1$-linearly independent. This is a very simple fact, and it is Proposition \ref{prop: linear independence} below. \\
$(3)$ On the other hand, by construction, our $N$ reals in $A(X,C)$ are the image of an integer matrix of size at most $C$. Siegel's lemma from the geometry of numbers shows that they must satisfy a non-trivial linear dependency of size about (at most) $(CN)^{\frac{n}{N-n}}$, where $n:=|X|$. For $N$ of the size $n \cdot \text{log}(n)$, this bound on the size of the linear dependency is no more than $C \cdot C_0$, as $n$ runs over all positive integers, for $C_0$ some universal positive constant. \\
$(4)$ Let us now take $D$ as in part $(2)$ relative to $C_1:=C \cdot C_0$. If now we witness (say) at least $100 \cdot D \cdot |X| \cdot \text{log}(|X|)$ dyadic intervals in $A(X,C)$, then by a greedy search, we can find about at least $|X| \cdot \text{log}(|X|)$ non-zero values in $A(X,C)$ pairwise spaced by about $D$. Now by step $(1)$ they are $C \cdot C_0$-linearly independent. On the other hand, in $(3)$ we have witnessed a non-zero $C \cdot C_0$-linear dependency between these numbers, which is a contradiction. 

\paragraph{Overview of the proof of Theorem \ref{thm: main2}}
The following overview is directly extracted from \emph{Aletheia}'s output, which gives a remarkably clear high level view of the steps. Here is \emph{Aletheia}'s summary:\\
\emph{Method Sketch:
Let $k=|X|$. We aim to bound $N(X,C)$, the number of intervals $[2^n, 2^{n+1})$ ($n \ge 0$) intersecting $A(X,C)$. We focus on the set $Y = A(X,C) \cap [1, \infty)$.\\
The upper bound is derived using a Gap Principle established via Siegel's Lemma. We define $R=2kC+1$. We prove that any sequence in $Y$ where consecutive terms have a ratio of at least $R$ (an $R$-separated sequence) has length at most $2k-1$. This is shown by contradiction: assuming a sequence of length $2k$, Siegel's Lemma implies a non-trivial integer linear relation among the elements with coefficients bounded by $2kC$. However, the $R$-separation property with $R>2kC$ forbids such a relation.\\
We then cover $Y$ by a union of at most $2k-1$ intervals, each of the form $[s, Rs)$. The number of dyadic intervals intersecting such an interval is bounded by $\lfloor \log_2 R \rfloor + 2$. This yields the upper bound $N(X,C) \le (2k-1)(\lfloor \log_2(2kC+1) \rfloor + 2)$. This bound is $O(k(\log k + \log C))$.}

We remark that Siegel's lemma is used in a different regime in the two proofs of Theorem \ref{thm: main1} and Theorem \ref{thm: main2}, respectively. In the former it directly controls combinations of uniformly bounded size, while in the latter it allows combinations of size $|X|$.

\paragraph{Further remarks}
The main difference between Proposition \ref{prop: function fields} and Theorem \ref{thm: main1} is the presence of the term $\text{log}(|X|)$. This leads us to the following question.  \\
\textbf{Question:} Can the bound in Theorem \ref{thm: main1} be improved to $|\text{Dyad}(A(X))|=O(|X|)$? 

It would be interesting to establish an asymptotically sharp bound. On the direction of improving the upper bound, we remark that, shortly after learning the agent's proof, Naser Sardari communicated to the authors that he can improve Theorem \ref{thm: main2} with the upper bound
$$|\textup{Dyad}(A(X,C))| \le |X|\text{log}(|X|)+ 2 \cdot |X|\cdot \text{log}(C)+1,
$$
using a different method. This will appear in a forthcoming preprint of his. 
  
\paragraph{Proof of Theorem \ref{thm: main1}}
Let us formalize the notion of linear independence explained above. 
\begin{definition}
Let $N$ be a positive integer and $C_1$ a positive real number. We say that a collection $y_1, \ldots, y_N$ of real numbers is $C_1$-linearly independent in case 
$$(y_1, \ldots, y_{N}) \cdot \underline{x} \neq 0,
$$
for all non-zero vectors $\underline{x}$ in $\mathbb{Z}^{N}$ such that
$$\infnorm{\underline{x}} \leq C_1.$$ 
They are said $C_1$-linearly dependent if they are not $C_1$-linearly independent. 
\end{definition}
The following fact can be viewed as a criterion for ``linear independence" of a set of real numbers over a set of coefficients with bounded size: it is the analogue of the observation that over $\mathbb{F}[T]$ a set of non-zero polynomials of different degrees must be linearly independent. 
\begin{proposition} \label{prop: linear independence}
Let $C_1$ be a positive real. Then there exists a positive integer $D$ such that the following holds. Let $N \geq 2$ be an integer. Let $(y_1, \ldots, y_{N})$ in $\mathbb{R}^{N}$ such that $y_1 \neq 0$ and for each $1 \leq i <N$, we have that $\textup{deg}(y_{i+1}) > D + \textup{deg}(y_i)$.

Then we have that $y_1, \ldots, y_{N}$ is $C_1$-linear independent. 
\begin{proof}
Let $\underline{x}$ be in $\mathbb{Z}^N$ such that  
$$(y_1, \ldots, y_{N}) \cdot \underline{x}=0.
$$
It will suffice to show that $x_N=0$, and then proceed by induction on $N$, which will give that $x_i=0$ for each $i \geq 2$. At that point, it must follow that $x_1=0$, since $y_1 \neq 0$ by assumption. 

To this end, notice that, as soon as $D$ is sufficiently large, we have the preliminary bound
$$|y_{i+1}| > 2 \cdot |y_i|,
$$
for all $1 \leq i <N$. This implies in particular that
$$|y_{N-1}| > \sum_{1 \leq i \leq N-2} |y_i|.
$$
On the other hand, when $D$ is sufficiently large, we also have the bound
$$|y_N| > 2 \cdot C_1 \cdot |y_{N-1}|.
$$
Thus, we deduce that, if $x_N$ is non-zero, since it is an integer, it must be that $|x_N \cdot y_N| \geq |y_N|$ and therefore
$$|x_N||y_N| \geq |y_N| > 2C_1 \cdot |y_{N-1}| \geq  C_1 \cdot (|y_{N-1}|+ \ldots +|y_1|) \geq |\sum_{1 \leq i \leq N-1} x_iy_i|,
$$
which directly contradicts
$$(y_1, \ldots, y_{N}) \cdot \underline{x}=0.
$$
This contradiction stemmed from the assumption that $x_N \neq 0$, and hence it must be that $x_N=0$. Now we can proceed inductively on the rest of the coefficients and find that they are all $0$, where in the base case we use that $y_1 \neq 0$ by assumption, thus yielding that $\underline{x}=0$. 
\end{proof}
\end{proposition}
We need the following application of Siegel's lemma. 
\begin{proposition} \label{prop: Siegel}
There exists an absolute constant $C_0>0$ such that the following holds. For all positive integers $n$ and any real number $C$ at least $1$, and any matrix $A$ with coefficients in $\mathbb{Z}$ with $\infnorm{A} \leq C$, $n$ columns and $N$ rows, with $n \cdot \textup{log}(N) \geq N \geq n \cdot \textup{log}(n) - 1$, one has a non-zero vector $\underline{x} \in \mathbb{Z}^{N}$ in the left-kernel of $A$ with
$$\infnorm{\underline{x}} \leq C_0 \cdot C. 
$$
\begin{proof}
Siegel's lemma tells us that there is a vector $\underline{x}$ of norm at most 
$$(\infnorm{A}N)^{\frac{n}{N-n}} \leq (C \cdot (n \cdot  \text{log}(n)))^{\frac{1}{\text{log}(n) - 1 - \frac{1}{n}}} \leq C \cdot (n \cdot  \text{log}(n))^{\frac{1}{\text{log}(n) - 1 - \frac{1}{n}}}.
$$
The last quantity in the product converges to $e$ as $n$ goes to $\infty$, in particular, it stays bounded, yielding a $C_0$ as desired. 
\end{proof}
\end{proposition}
We are now in a position to prove the following. 

\begin{proof}[Proof of Theorem~\ref{thm: main1}]
Let $C_0$ be as in Proposition \ref{prop: Siegel}. Take now any choice of $D$ coming from Proposition \ref{prop: linear independence}, where we place $C_1:=C_0 \cdot C$. Suppose now that there is a non-empty set $X$ of positive reals with $|\text{Dyad}(A(X,C))| \geq 100 \cdot D \cdot |X| \cdot \text{log}(|X|)$. That means, by a greedy search, that we can arrange to have a positive integer $|X| \cdot \text{log}(|X|) \geq N \geq |X| \cdot \text{log}(|X|) - 1$ and a vector
$$(y_1, \ldots, y_{N}) \in A(X,C)^N,
$$
such that $y_1 \neq 0$ and for each $1 \leq i <N$, we have that $\textup{deg}(y_{i+1}) > D + \textup{deg}(y_i)$. In virtue of our choice of $D$ and of Proposition \ref{prop: linear independence}, we see that $\underline{y}$ cannot be orthogonal to any non-zero integer vector of $\infnorm{\cdot}$-norm at most $C_0 \cdot C$. On the other hand, by construction, $\underline{y}$ is in the image of an integer matrix $A$ having $|X|$ columns, $|X| \text{log}(|X|) \geq N \geq |X| \text{log}(|X|) -  1$ rows and $\infnorm{A} \leq C$. Hence, the element provided by Proposition \ref{prop: Siegel} in the left kernel is a non-zero integer vector of norm at most $C \cdot C_0$, which is orthogonal to anything in the image of $A$, since the kernel of the transpose equals the orthogonal to the image. This gives a contradiction. Therefore, it must be that
$$|\text{Dyad}(A(X, C))| \leq 100 \cdot D \cdot |X| \cdot \text{log}(|X|),
$$
which is the desired conclusion. 
\end{proof}

\paragraph{Proof of Theorem \ref{thm: main2}}
\emph{Aletheia} was prompted to search for elements in $A(X,C) \cap [1, \infty)$. Clearly whatever upper bound $B$ one finds for the number of dyadic intervals touched by $A(X,C) \cap [1,\infty)$, then provides a bound of the form $2B+1$ for $\text{Dyad}(A(X,C))$. 

The text that follows is directly extracted from the output of \emph{Aletheia}. 

We utilize Siegel's Lemma to establish a Gap Principle.

\begin{lemma}[Siegel's Lemma]\label{lem:siegel}
Let $B$ be an $M \times N$ matrix with integer entries, with $N>M>0$. Let $H \ge 1$ be an upper bound for the absolute values of the entries of $B$. Then there exists a non-zero integer vector $c \in \mathbb{Z}^N \setminus \{0\}$ such that $Bc=0$ and
$$ \|c\|_\infty \le (N H)^{M/(N-M)}.$$
\end{lemma}
\begin{proof}
See, for example, M. Hindry and J. H. Silverman, Diophantine Geometry: An Introduction, Springer, 2000, Part D, Lemma 4.1.
\end{proof}

A sequence of positive real numbers $(z_i)$ is called $R$-separated if $z_{i+1}/z_i \ge R$ for all $i$.

\begin{lemma}[Gap Principle]\label{lem:gap}
Let $k \ge 1, C \ge 1$. Let $R = 2kC+1$. Any $R$-separated sequence of distinct positive elements in $A(X,C)$ has length $m \le 2k-1$.
\end{lemma}
\begin{proof}
Let $C' = 2kC$. Note that $R=C'+1$. Suppose for contradiction that there exists an $R$-separated sequence $0 < z_1 < z_2 < \dots < z_m$ in $A(X,C)$ with $m=2k$.

For each $i=1, \dots, m$, $z_i \in A(X,C)$, so $z_i = \sum_{j=1}^k a_{ij} x_j$ with $a_{ij} \in \mathbb{Z}$ and $|a_{ij}| \le C$. Let $v_i=(a_{i1}, \dots, a_{ik})$.

Let $B$ be the $k \times m$ matrix whose columns are the vectors $v_i$. The entries of $B$ are bounded by $H=C$. We apply Lemma \ref{lem:siegel} with $M=k$ and $N=m=2k$. Since $N>M$, there exists a non-zero integer vector $c=(c_1, \dots, c_m) \in \mathbb{Z}^m$ such that $B c = 0$. The bound on the coefficients is
$$ \|c\|_\infty \le (N H)^{M/(N-M)} = (2kC)^{k/(2k-k)} = 2kC = C'. $$
The condition $B c = 0$ implies a linear relation among the elements $z_i$:
$$ \sum_{i=1}^m c_i z_i = \sum_{i=1}^m c_i \sum_{j=1}^k a_{ij} x_j = \sum_{j=1}^k x_j \left(\sum_{i=1}^m c_i a_{ij}\right) = 0. $$

Let $J$ be the largest index such that $c_J \neq 0$. Then $c_J z_J = - \sum_{i=1}^{J-1} c_i z_i$. Taking absolute values (and noting $z_i>0$):
$$ |c_J| z_J = \left| \sum_{i=1}^{J-1} c_i z_i \right| \le \sum_{i=1}^{J-1} |c_i| z_i \le C' \sum_{i=1}^{J-1} z_i. $$
Since the sequence $(z_i)$ is $R$-separated and increasing, $z_i \le R^{-(J-1-i)} z_{J-1}$ for $i \le J-1$.
$$ \sum_{i=1}^{J-1} z_i \le z_{J-1} \sum_{l=0}^{J-2} R^{-l} < z_{J-1} \sum_{l=0}^{\infty} R^{-l} = z_{J-1} \frac{R}{R-1}. $$
Also, by the separation property, $z_J \ge R z_{J-1}$. Combining these inequalities:
$$ |c_J| R z_{J-1} \le |c_J| z_J < C' z_{J-1} \frac{R}{R-1}. $$
Since $z_{J-1}>0$ and $R>0$, we can divide by $R z_{J-1}$:
$$ |c_J| < \frac{C'}{R-1}. $$
We chose $R=C'+1$, so $R-1=C'$.
$$ |c_J| < \frac{C'}{C'} = 1. $$
This contradicts the fact that $c_J$ is a non-zero integer. Therefore, the length of the sequence must be $m \le 2k-1$.
\end{proof}

\begin{lemma}\label{lem:interval_count}
Let $I=[a, b)$ be an interval with $1 \le a < b$. The number of dyadic intervals $[2^n, 2^{n+1})$ with $n \ge 0$ that intersect $I$, denoted by $D(I)$, is bounded by
$$D(I) \le \lfloor \log_2(b/a) \rfloor + 2.$$
\end{lemma}
\begin{proof}
Since $a \ge 1$, the exponents $n$ must be non-negative. The smallest exponent $n_{min}$ is $\lfloor \log_2 a \rfloor$. The largest exponent $n_{max}$ must satisfy $2^{n_{max}} < b$, so $n_{max} = \lceil \log_2 b \rceil - 1$.
The count is $D(I) = n_{max}-n_{min}+1 = \lceil \log_2 b \rceil - \lfloor \log_2 a \rfloor$.
Let $R=b/a$. Let $r = \log_2 R$ and $x = \log_2 a$. $D(I) = \lceil r+x \rceil - \lfloor x \rfloor$.
Let $I_r = \lfloor r \rfloor, f_r = \{r\}$, and $I_x = \lfloor x \rfloor, f_x = \{x\}$.
$D(I) = \lceil I_r+f_r+I_x+f_x \rceil - I_x = I_r + \lceil f_r+f_x \rceil$.
Since $0 \le f_r, f_x < 1$, we have $0 \le f_r+f_x < 2$. Thus $\lceil f_r+f_x \rceil \in \{0, 1, 2\}$.
Therefore, $D(I) \le I_r + 2 = \lfloor \log_2 R \rfloor + 2$.
\end{proof}

We now prove the upper bound in Theorem \ref{thm: main2}.

\begin{proof}[Proof of the Upper Bound]
Let $Y = A(X,C) \cap [1, \infty)$. If $Y$ is empty, $N(X,C)=0$. Assume $Y$ is non-empty. $Y$ is finite.
Let $R = 2kC+1$. We construct a maximal $R$-separated subsequence of $Y$ greedily.
Let $s_1 = \min Y$. Since $Y \subset [1, \infty)$, $s_1 \ge 1$.
For $j \ge 1$, if $s_j$ is defined, let $Y_j = \{y \in Y : y \ge R s_j\}$. If $Y_j$ is non-empty, define $s_{j+1} = \min Y_j$. Otherwise, the sequence terminates.
Let the sequence be $S = \{s_1, \dots, s_m\}$. By construction, $s_{j+1} \ge R s_j$. By Lemma \ref{lem:gap}, $m \le 2k-1$.

We show that $Y$ is covered by the union of the intervals $J_j = [s_j, R s_j)$ for $j=1, \dots, m$.
Let $y \in Y$. Since $s_1=\min Y$, $y \ge s_1$. Let $J$ be the largest index such that $s_J \le y$.
If $J=m$. The sequence terminated because $Y_m$ is empty. Thus, any $y' \in Y$ with $y' \ge s_m$ must satisfy $y' < R s_m$. So $y \in J_m$.
If $J<m$. Then $s_{J+1}$ exists, and $y < s_{J+1}$ by the maximality of $J$. If we had $y \ge R s_J$, then $y \in Y_J$. By definition, $s_{J+1} = \min Y_J \le y$, a contradiction. Thus $y < R s_J$. So $y \in J_J$.
Therefore, $Y \subset \bigcup_{j=1}^m J_j$.

The total number of dyadic intervals intersecting $Y$ is bounded by the sum of the counts for each $J_j$. By Lemma \ref{lem:interval_count}, since $s_j \ge 1$ and the ratio for $J_j$ is $R$, the number of dyadic intervals intersecting $J_j$ is at most $\lfloor \log_2 R \rfloor + 2$.
$$ N(X,C) \le \sum_{j=1}^m D(J_j) \le m (\lfloor \log_2 R \rfloor + 2) \le (2k-1) (\lfloor \log_2(2kC+1) \rfloor + 2). $$
\end{proof}

This upper bound is $O(k \log(kC)) = O(k(\log k + \log C))$.

\section{Proof of Theorem \ref{thm:rmc:iterations}} \label{app:thm:rmc:iterations}
\begin{proof}
    Let $\envpol^{0}, \envpol^{1}, \ldots, \envpol^{T}$ be the sequence of policies generated by the algorithm.
    
    \noindent\textbf{1. Finite Discrepancy Set.}
    For a state $s$ with nominal transition vector $\nominal_{s}$ and uncertainty radius $\radius_s$, the realized transition probability $P=\uncert^\envpol_{s}$ in any policy $\envpol$ results from redistributing mass under the $\linf$ constraint, according to Algorithm \ref{algo:homotopy}. Given the properties of such policies, discussed in Section~\ref{sec:algo}, we define $X_s$ as follows so that $\uncert^\envpol_{s,s'} \in A(X_s)$ for all $s,s' \in \states$:

    \[
    X_s = \{\nominal_{s,s'} \;\mid\; s' \in \States\} \;\cup\; \{1\} \; \cup \; \left\{ \radius(s), 2\cdot\radius(s),\cdots,n\cdot\radius(s)\right\}.
    \]
    Given the properties of homotopic policies generated by Algorithm~\ref{algo:homotopy}, it can be verified that all possible values of $\uncert^*_{s,s'} - \uncert^{\envpol}_{s,s'}$ and $\uncert^\envpol_{s,s'} - \uncert^{*}_{s,s'}$ are included in $A(X_s)$ when we set the constant $c=2$. Note that $|X_s| =2n+1$. Lemma~\ref{lemma:combinatorial_lemma} implies that $A(X_s)$ has polynomially many different MSBs:
    \[
        \mathrm{Deg}(A(X_s)) \in \mathcal{O}(|X_s| \log |X_s|) = \mathcal{O}(n \log n).
    \]
    So, $\uncert^*_{s,s'} - \uncert^{\envpol}_{s,s'}$ and $\uncert^\envpol_{s,s'} - \uncert^{*}_{s,s'}$ also can take polynomially many MSBs.
    
    \noindent\textbf{2. Convergence Rate.}
    Let 
    \[
    f_t(s,s',s'') = \min\left( \uncert^*_{s,s'} - \uncert^{\envpol^{t}}_{s,s'}, \;
                     \uncert^{\envpol^{t}}_{s,s''} - \uncert^*_{s,s''} \right) (\valvector^*_{s'} - \valvector^*_{s''})
    \]
    be the potential value at iteration $t$. Let $(s_t,s'_t,s''_t) = \argmax_{s,s',s''} f_{\envpol^{t}}(s,s',s'')$ be the maximizing triple. Lemma~\ref{lemma:subaction_not_repeat} states that for $L=\log_\discountfactor\left(\frac{1 - \discountfactor}{2n^2}\right)$, the mass transfer quantity associated with this triple decreases by a factor of at least 2 after $L$ steps. In terms of logarithmic degree, the most significant bit of this quantity decreases by at least 1 every $L$ steps.

    \textbf{3. Total Complexity.}
    There are at most $n^3$ distinct triples $(s,s',s'')$. For any specific triple, the mass transfer value can take at most $\mathcal{O}(n \log n)$ different logarithmic scales (from Step 1) before it vanishes or the triple is no longer active. Since each scale reduction requires at most $L$ steps, we bound the total number of steps by:
    \[
         \text{Total Iterations} \leq n^3 \times \mathcal{O}(n \log n) \times L = \mathcal{O}(n^4 \log n \cdot L).
    \]

    We substitute the definition of $L$ and omit constant factors to obtain:
    \[
        \mathcal{O}\left(n^4 \log n \cdot L\right)
        = \mathcal{O}\left(n^4 \log n \cdot \frac{\log\left(\frac{1 - \discountfactor}{n^2}\right)}{\log{\discountfactor}} \right).
    \]
\end{proof}

\section{Proof of Lemma \ref{lemma:contractive_value_iteration_rmdp}} \label{app:lemma:contractive_value_iteration_rmdp}
\begin{proof}
    \begin{compactitem}
        \item To prove contraction, fix $s \in \States$. Let $a^* = \argmin_{a \in \Actions}\;{\max_{\pvector \in \uncertaintyset_{s,a}}{\pvector^\top \valvector}}$ and $\pvector' = \argmax_{\pvector \in \uncertaintyset_{s,a^*} } \pvector^\top \uvector$. Then:
    \begin{align*}
        (\Bellman\uvector)_s - (\Bellman\valvector)_s
        &= \discountfactor \left( \min_{a \in \Actions} \max_{\pvector \in \uncertaintyset_{s,a}} \pvector^\top \uvector - \min_{a \in \Actions} \max_{\pvector \in \uncertaintyset_{s,a}} \pvector^\top \valvector \right) & (\text{By Definition}) \\
        &\leq \discountfactor \left( \max_{\pvector \in \uncertaintyset_{s,a^*}} \pvector^\top \uvector - \max_{\pvector \in \uncertaintyset_{s,a^*}} \pvector^\top \valvector \right) & (\text{By definition of } a^*) \\
        &\leq \discountfactor \pvector'^\top( \uvector - \valvector) & (\text{By definition of } \pvector') \\
        &\leq \discountfactor \infnorm{\uvector - \valvector} & (\text{Since } \pvector' \text{ is a distribution})
    \end{align*}
    Symmetry yields the absolute value bound. 
    \item \begin{align*}
        (\Bellman \valvector^{\agentpol})_s
        & = \costvector_s + \discountfactor \min_{a \in \Actions} \max_{p \in \uncertaintyset_{s,a}} p^\top \valvector^{\agentpol} & (\text{Definition of } \Bellman) \\
        & \leq \costvector_s + \discountfactor \max_{p \in \uncertaintyset_{s,\agentpol(s)}} p^\top \valvector^{\agentpol} & (\text{Restricting min to } \agentpol(s)) \\
        & = \valvector^{\agentpol}_s & (\text{Definition of } \valvector^{\agentpol})
    \end{align*}
    \item The fixed-point property follows directly from the Banach fixed-point theorem.
    \end{compactitem}
\end{proof}

\section{Proof of Lemma \ref{lemma:consecutive_policy_iteration_bound_rmdp}} \label{app:lemma:consecutive_policy_iteration_bound_rmdp}
\begin{proof}
    \begin{compactitem}
        \item \begin{align*}
        \valvector^{t} - \valvector^{t + 1}
        & = \valvector^t - (\identitymatrix - \gamma \trans^{t+1})^{-1} \costvector & (\text{By Equation \ref{eq:value_definition_rmdp}}) \\
        & = (\identitymatrix - \gamma \trans^{t+1})^{-1} (\identitymatrix - \gamma \trans^{t+1}) \left[ \valvector^t - (\identitymatrix - \gamma \trans^{t+1})^{-1} \costvector \right] & (\text{Multiplying by } \identitymatrix) \\
        & = (\identitymatrix - \gamma \trans^{t+1})^{-1} \left[ (\identitymatrix - \gamma \trans^{t+1}) \valvector^t - \costvector \right] & (\text{Expanding}) \\
        & = (\identitymatrix - \gamma \trans^{t+1})^{-1} \left[ \valvector^t - (\costvector + \gamma \trans^{t+1} \valvector^t) \right] & (\text{Rearranging terms}) \\
        & \succcurlyeq (\identitymatrix - \gamma \trans^{t+1})^{-1} \left[ \valvector^t - \Bellman \valvector^t \right] & (\text{By Definition of $\Bellman$})
    \end{align*}
    By Lemma \ref{lemma:contractive_value_iteration_rmdp}, $\valvector^t -\Bellman \valvector^t \succcurlyeq \mathbf{0}$. Since $(\identitymatrix - \gamma \trans^{t+1})^{-1} \succcurlyeq \mathbf{0}$, the result follows.
    \item We first observe that $\costvector + \gamma \trans^{t+1} \valvector^t \preccurlyeq \Bellman \valvector^t$: for each state $s$, since $\trans^{t+1}_s \in \mathcal{P}(s, \agentpol^{t+1}(s))$ is one feasible point,
\begin{equation*}
    \trans^{t+1}_s \valvector^t \;\leq\; \max_{\mathbf{p} \in \mathcal{P}(s,\, \agentpol^{t+1}(s))} \mathbf{p}^\top \valvector^t \;=\; \min_{a \in \mathcal{A}} \max_{\mathbf{p} \in \mathcal{P}(s, a)} \mathbf{p}^\top \valvector^t,
\end{equation*}
where the equality uses that $\agentpol^{t+1}(s)$ is the greedy improvement w.r.t.\ $\valvector^t$. Multiplying by $\gamma$ and adding $\costvector_s$ yields the claim.

We now show $\valvector^{t+1} \preccurlyeq \Bellman \valvector^t$ as follows:
    \begin{align*}
        \Bellman \valvector^t - \valvector^{t+1}
        & \succcurlyeq \costvector + \gamma \trans^{t+1} \valvector^t - \valvector^{t+1} & (\costvector + \gamma \trans^{t+1} \valvector^t \preccurlyeq \Bellman \valvector^t) \\
        & = (\identitymatrix - \gamma \trans^{t+1})\valvector^{t+1} + \gamma \trans^{t+1} \valvector^t - \valvector^{t+1} & (\text{Substituting } \costvector) \\
        & = \valvector^{t+1} - \gamma \trans^{t+1} \valvector^{t+1} + \gamma \trans^{t+1} \valvector^t - \valvector^{t+1} & (\text{Expanding}) \\
        & = \gamma \trans^{t+1} (\valvector^t - \valvector^{t+1}) & (\text{Simplifying})
    \end{align*}
    Since $\discountfactor > 0$, $\trans^{t+1} \geq 0$, and $\valvector^t \succcurlyeq \valvector^{t+1}$ (proved above), we conclude $\valvector^{t+1} \preccurlyeq \Bellman \valvector^{t}$.

    Now, to prove the second statement of the lemma we use induction on $t$. The base case $t=0$ holds trivially. For $t > 0$:
    \begin{align*}
        \infnorm{\valvector^t - \valvector^*}
        & \leq \infnorm{\Bellman \valvector^{t-1} - \valvector^*} & ( \valvector^* \preccurlyeq \valvector^t \preccurlyeq \Bellman\valvector^{t-1}) \\
        & = \infnorm{\Bellman \valvector^{t-1} - \Bellman \valvector^*} & (\text{Since } \valvector^* \text{ is the fixed point}) \\
        & \leq \gamma \infnorm{\valvector^{t-1} - \valvector^*} & (\text{By contraction in Lemma \ref{lemma:contractive_value_iteration_rmdp}}) \\
        & \leq \gamma^t \infnorm{\valvector^0 - \valvector^*} & (\text{By induction})
    \end{align*}
    \end{compactitem}
\end{proof}

\section{Proof of Lemma \ref{lemma:lowerbound_potential_rmdp}} \label{app:lemma:lowerbound_potential_rmdp}
\begin{proof}
    \begin{compactitem}
        \item \textbf{(Lower-bound)}
    \begin{align*}
        \valvector^\agentpol_s - \valvector^*_s
        &= \left( \costvector_s + \discountfactor \max_{\pvector \in \uncertaintyset_{s,a}}{ \pvector^\top \valvector^\agentpol} \right)
         - \left( \costvector_s + \discountfactor \max_{\pvector \in \uncertaintyset_{s,\agentpol^*(s)}}{ \pvector^\top \valvector^*} \right)
        && (\text{By Equation~\ref{eq:value_definition_rmdp}})
        \\
        & = \discountfactor \left(\max_{\pvector \in \uncertaintyset_{s,a}}{\pvector^\top \valvector^\agentpol} - \max_{\pvector \in \uncertaintyset_{s,\agentpol^*(s)}}{\pvector^\top \valvector^*}\right) & &(\text{Simplification})
        \\
        &\geq \discountfactor \left(\max_{\pvector \in \uncertaintyset_{s,a}}{\pvector^\top \valvector^*} - \max_{\pvector \in \uncertaintyset_{s,\agentpol^*(s)}}{\pvector^\top \valvector^*}\right)
        && (\text{Since } \valvector^\agentpol \succcurlyeq \valvector^*) \\
        &= \gamma f(s,a)
        && (\text{By Definition})
    \end{align*}
    \item \textbf{(Upper-bound)} Assume $\envpol$ is environment's best response to $\agentpol$ and $\envpol^*$ is environment's best response to $\agentpol^*$. Denote the transition matrix when fixing policies by $\trans^{\agentpol,\envpol}$ and $\trans^{*}$. We begin by expanding the value difference using the value definition from Equation \ref{eq:value_definition_rmdp}:

    \begin{align*}
        \valvector^\agentpol - \valvector^*
        &= (\identitymatrix - \discountfactor\trans^{\agentpol,\envpol})^{-1} \costvector - (\identitymatrix - \discountfactor\trans^*)^{-1} \costvector \\
        &= \bigg( (\identitymatrix - \discountfactor\trans^{\agentpol,\envpol})^{-1} - (\identitymatrix - \discountfactor\trans^*)^{-1} \bigg) \costvector
    \end{align*}

    We apply the matrix identity $A^{-1} - B^{-1} = A^{-1}(B - A)B^{-1}$ to the term inside the norm:

    \begin{align*}
        \valvector^\agentpol - \valvector^* &= (\identitymatrix - \discountfactor\trans^{\agentpol,\envpol})^{-1} \bigg( (\identitymatrix - \discountfactor\trans^*) - (\identitymatrix - \discountfactor\trans^{\agentpol,\envpol}) \bigg) (\identitymatrix - \discountfactor\trans^*)^{-1} \costvector \\
        &= \discountfactor \, (\identitymatrix - \discountfactor\trans^{\agentpol,\envpol})^{-1} (\trans^{\agentpol,\envpol} - \trans^*) (\identitymatrix - \discountfactor\trans^*)^{-1} \costvector
    \end{align*}

    Replacing $(\identitymatrix - \discountfactor\trans^*)^{-1} \costvector$ by $\valvector^*$ (based on Equation~\ref{eq:value_definition_rmdp}) we have the following equation:
\[
    \valvector^\agentpol - \valvector^* \;=\; \discountfactor (\identitymatrix - \discountfactor\trans^{\agentpol,\envpol})^{-1} (\trans^{\agentpol,\envpol} - \trans^*) \valvector^*.
\]
We first bound the vector $(\trans^{\agentpol,\envpol} - \trans^*)\valvector^*$ component-wise. Fix $s \in \States$. Since $\envpol$ is the environment's best response to $\agentpol$, we have $\trans^{\agentpol,\envpol}_s \in \uncertaintyset_{s,\agentpol(s)}$, so:
\begin{align*}
    \bigl( (\trans^{\agentpol,\envpol} - \trans^*) \valvector^* \bigr)_s
    &= \trans^{\agentpol,\envpol}_s \valvector^* - \trans^*_s \valvector^*
    && (\text{Definition of $\trans^{\agentpol,\envpol}_s$ and $\trans^*_s$}) \\
    &\leq \max_{\pvector \in \uncertaintyset_{s,\agentpol(s)}} \pvector^\top \valvector^* - \max_{\pvector \in \uncertaintyset_{s,\agentpol^*(s)}} \pvector^\top \valvector^*
    && (\text{Taking the max over $\uncertaintyset_{s,\agentpol(s)}$}) \\
    &= f(s, \agentpol(s)).
    && (\text{Definition of $f$})
\end{align*}
Define the vector $\mathbf{f}^\agentpol \in \R^n$ by $\mathbf{f}^\agentpol_s := f(s, \agentpol(s))$. The componentwise bound above gives
\[
    (\trans^{\agentpol,\envpol} - \trans^*) \valvector^* \;\preccurlyeq\; \mathbf{f}^\agentpol.
\]
Since $\trans^{\agentpol,\envpol}$ is stochastic, the matrix $(\identitymatrix - \discountfactor\trans^{\agentpol,\envpol})^{-1} = \sum_{i \geq 0} (\discountfactor \trans^{\agentpol,\envpol})^i$ has only nonnegative entries, so multiplying both sides preserves the inequality. Combined with $\valvector^\agentpol \succcurlyeq \valvector^*$, we obtain the nonnegative bounding chain
\[
    \mathbf{0} \;\preccurlyeq\; \valvector^\agentpol - \valvector^* \;\preccurlyeq\; \discountfactor (\identitymatrix - \discountfactor\trans^{\agentpol,\envpol})^{-1} \mathbf{f}^\agentpol.
\]
Taking $\infnorm{\cdot}$ preserves inequalities between nonnegative vectors, so:
\begin{align*}
    \infnorm{\valvector^\agentpol - \valvector^*}
    &\leq \discountfactor \, \infnorm{(\identitymatrix - \discountfactor\trans^{\agentpol,\envpol})^{-1} \mathbf{f}^\agentpol}
    && (\text{Bounding chain above}) \\
    &\leq \discountfactor \, \infnorm{(\identitymatrix - \discountfactor\trans^{\agentpol,\envpol})^{-1}} \cdot \infnorm{\mathbf{f}^\agentpol}
    && (\text{Submultiplicativity of $\linf$ norm}) \\
    &\leq \frac{\discountfactor}{1 - \discountfactor} \cdot \infnorm{\mathbf{f}^\agentpol}
    && (\text{Neumann series; $\trans^{\agentpol,\envpol}$ stochastic}) \\
    &= \frac{\discountfactor}{1 - \discountfactor} \cdot \max_{s \in \States} f(s, \agentpol(s))
    && (\text{Definition of $\linf$ norm and $\mathbf{f}^\agentpol$}) \\
    &= \frac{\discountfactor}{1 - \discountfactor} \, f(\hat{s}, \agentpol(\hat{s})).
    && (\text{Definition of $\hat{s}$})
\end{align*}
This completes the proof.

    \end{compactitem}
\end{proof}

\section{Proof of Lemma \ref{lemma:action_elimination}} \label{app:lemma:action_elimination}
\begin{proof} 
    Assume for the sake of contradiction that there exists $k > l + L$ such that $\agentpol^k(\hat{s}) = \agentpol^l(\hat{s})$. Since the value vectors are non-increasing, $\valvector^{\agentpol^l} \ge \valvector^{\agentpol^k}$. Applying Lemma \ref{lemma:lower_upper_combination_rmdp}:
    \[
    \infnorm{\valvector^{\agentpol^k} - \valvector^*} \ge (1 - \gamma) \infnorm{\valvector^{\agentpol^l} - \valvector^*}.
    \]
    However, from the global convergence rate (Lemma \ref{lemma:consecutive_policy_iteration_bound_rmdp}), we know:
    \[
    \infnorm{\valvector^{\agentpol^k} - \valvector^*} \le \gamma^{k-l} \infnorm{\valvector^{\agentpol^l} - \valvector^*}.
    \]
    We note that if $\|\mathbf{v}^{\sigma^l} - \mathbf{v}^*\|_\infty = 0$, then $\sigma^l$ is optimal, and the algorithm terminates at iteration $l$. Since the sequence continues to $k > l$, we must have $\|\mathbf{v}^{\sigma^l} - \mathbf{v}^*\|_\infty > 0$. We can therefore divide both inequalities by this strictly positive term to obtain:
    % Combining these inequalities implies $1 - \gamma \le \gamma^{k-l}$. Taking the logarithm with base $\gamma$ (reversing the inequality as $\gamma < 1$) yields:
    \[
    \log_\gamma(1-\gamma) \ge k - l \implies L \ge k - l.
    \]
    This contradicts the assumption that $k > l + L$.
\end{proof}

\section{Proof of Theorem \ref{thm:rmdp:iterations}} \label{app:thm:rmdp:iterations}

\begin{proof}
    For each iteration $\ell$, let $\hat{s}_\ell := \argmax_{s \in \States} f(s, \agentpol^\ell(s))$ and call $C_\ell := (\hat{s}_\ell, \agentpol^\ell(\hat{s}_\ell))$ the \emph{critical pair} at $\ell$. We show that each pair $(s, a) \in \States \times \Actions$ can be critical only inside a window of length at most $L$.

    Suppose a fixed pair $(s, a)$ is critical at two iterations $\ell_1 < \ell_2$, i.e., $C_{\ell_1} = C_{\ell_2} = (s, a)$. Then $\hat{s}_{\ell_1} = s$ and $\agentpol^{\ell_1}(s) = \agentpol^{\ell_2}(s) = a$. Applying Lemma~\ref{lemma:action_elimination} at $\ell_1$ (so that $\hat{s}$ in the lemma equals $s$) gives $\agentpol^k(s) \neq a$ for all $k > \ell_1 + L$. Hence $\ell_2 \leq \ell_1 + L$.

    Since every iteration has a critical pair and there are at most $n \cdot m$ distinct pairs, each contributing a window of at most $L$ iterations, the total number of iterations $T$ is bounded by:
    \[
    T \leq n \cdot m \cdot L = n \cdot m \cdot \log_\gamma(1 - \gamma) = n \cdot m \cdot \frac{\log(1 - \discountfactor)}{\log \discountfactor}.
    \]
\end{proof}